\documentclass[accepted]{uai2023} % after acceptance, for a revised
\usepackage[american]{babel}
\usepackage{natbib} % has a nice set of citation styles and commands
\usepackage{mathtools} % amsmath with fixes and additions
\usepackage{booktabs} % commands to create good-looking tables
\usepackage{tikz} % nice language for creating drawings and diagrams

\newcommand{\R}[1]{\text{Regret}(#1)}
\newcommand{\gap}[1]{\text{gap}_{#1}}

\newcommand{\image}[0]{\text{Im}}
\newcommand{\algname}[0]{\text{ReLEX}\xspace}
\usepackage{xspace} 

\RequirePackage{amsmath}
\RequirePackage{amssymb}
\RequirePackage{amsthm}
\RequirePackage{bm} 
\RequirePackage{url}
\usepackage{multirow}
\usepackage{natbib}
\usepackage{graphicx}
\usepackage{subfigure}
\usepackage{makecell}
\usepackage{booktabs}
\usepackage{array}
\usepackage{url}
\usepackage{algorithm}
\usepackage{algorithmic}
\usepackage{dsfont}

\newcommand{\la}{\left\langle}
\newcommand{\ra}{\right\rangle}

\let\hat\widehat
\let\tilde\widetilde

\newcommand{\db}{\mathbf{d}}

\newcommand{\vb}{\mathbf{v}}

\newcommand{\xb}{\mathbf{x}}
\newcommand{\yb}{\mathbf{y}}
\newcommand{\zb}{\mathbf{z}}

\newcommand{\Ab}{\mathbf{A}}
\newcommand{\Bb}{\mathbf{B}}
\newcommand{\Cb}{\mathbf{C}}
\newcommand{\Db}{\mathbf{D}}

\newcommand{\Ib}{\mathbf{I}}

\newcommand{\Kb}{\mathbf{K}}

\newcommand{\Mb}{\mathbf{M}}

\newcommand{\Qb}{\mathbf{Q}}

\newcommand{\Ub}{\mathbf{U}}

\newcommand{\Xb}{\mathbf{X}}

\newcommand{\cA}{\mathcal{A}}

\newcommand{\cD}{\mathcal{D}}
\newcommand{\cE}{\mathcal{E}}
\newcommand{\cF}{\mathcal{F}}

\newcommand{\cH}{\mathcal{H}}

\newcommand{\cM}{\mathcal{M}}
\newcommand{\cN}{\mathcal{N}}
\newcommand{\cO}{\mathcal{O}}

\newcommand{\cS}{{\mathcal{S}}}

\newcommand{\cZ}{\mathcal{Z}}

\newcommand{\EE}{\mathbb{E}}

\newcommand{\PP}{\mathbb{P}}

\newcommand{\RR}{\mathbb{R}}

\newcommand{\bepsilon}{\bm{\epsilon}}

\newcommand{\btheta}{\bm{\theta}}

\newcommand{\bphi}{\bm{\phi}}

\newcommand{\bpsi}{\bm{\psi}}

\newcommand{\bLambda}{\bm{\Lambda}}

\newcommand{\bPhi}{\bm{\Phi}}
\newcommand{\bPsi}{\bm{\Psi}}

\newcommand{\argmin}{\mathop{\mathrm{argmin}}}
\newcommand{\argmax}{\mathop{\mathrm{argmax}}}

\newcommand{\tr}{\mathop{\mathrm{tr}}}

\DeclareMathOperator{\ind}{\mathds{1}}  % Indicator

\newcommand*{\zero}{{\bm 0}}

\newcommand{\diag}{{\rm diag}}

\newtheoremstyle{mytheoremstyle} % name
    {\topsep}                    % Space above
    {\topsep}                    % Space below
    {\normalfont}                   % Body font
    {}                           % Indent amount
    {\bfseries}                   % Theorem head font
    {.}                          % Punctuation after theorem head
    {.5em}                       % Space after theorem head
    {}  % Theorem head spec (can be left empty, meaning ‘normal’)

\theoremstyle{mytheoremstyle}

\ifx\BlackBox\undefined
\newcommand{\BlackBox}{\rule{1.5ex}{1.5ex}}  % end of proof
\fi

\ifx\QED\undefined
\def\QED{~\rule[-1pt]{5pt}{5pt}\par\medskip}
\fi

\ifx\proof\undefined
\newenvironment{proof}{\par\noindent{\bf Proof\ }}{\hfill\BlackBox\\[2mm]}
\fi

\ifx\theorem\undefined
\newtheorem{theorem}{Theorem}
\fi
\ifx\example\undefined
\newtheorem{example}[theorem]{Example}
\fi
\ifx\property\undefined
\newtheorem{property}{Property}
\fi
\ifx\lemma\undefined
\newtheorem{lemma}[theorem]{Lemma}
\fi
\ifx\proposition\undefined
\newtheorem{proposition}[theorem]{Proposition}
\fi
\ifx\remark\undefined
\newtheorem{remark}[theorem]{Remark}
\fi
\ifx\corollary\undefined
\newtheorem{corollary}[theorem]{Corollary}
\fi
\ifx\definition\undefined
\newtheorem{definition}[theorem]{Definition}
\fi
\ifx\conjecture\undefined
\newtheorem{conjecture}[theorem]{Conjecture}
\fi
\ifx\fact\undefined
\newtheorem{fact}[theorem]{Fact}
\fi
\ifx\claim\undefined
\newtheorem{claim}[theorem]{Claim}
\fi
\ifx\assumption\undefined
\newtheorem{assumption}[theorem]{Assumption}
\fi
\ifx\condition\undefined
\newtheorem{condition}[theorem]{Condition}
\fi
\numberwithin{equation}{section}
\numberwithin{theorem}{section}

\title{Provably Efficient Representation Selection in Low-rank Markov Decision Processes: From Online to Offline RL}

\author[1]{\href{mailto:<weightzero@ucla.edu>?Subject=Your UAI 2023 paper}{Weitong~Zhang}{}}
\author[1]{\href{mailto:<jiafanhe19@ucla.edu>?Subject=Your UAI 2023 paper}{Jiafan~He}{}}
\author[1]{\href{mailto:<drzhou@cs.ucla.edu>?Subject=Your UAI 2023 paper}{Dongruo~Zhou}{}}
\author[2,3]{\href{mailto:<amyzhang@fb.com>?Subject=Your UAI 2023 paper}{Amy~Zhang}{}}
\author[1]{\href{mailto:<qgu@cs.ucla.edu>?Subject=Your UAI 2023 paper}{Quanquan~Gu}{}}
\affil[1]{%
    Department of Computer Science\\
    University of California, Los Angeles\\
    California, USA
}
\affil[2]{%
    Department of Electrical and Computer Engineering\\
    University of Texas at Austin\\
    Texas, USA
}
\affil[3]{%
    Facebook AI Research
  }
  
\begin{document}
\maketitle
\begin{abstract}
The success of deep reinforcement learning (DRL) lies in its ability to learn a representation that is well-suited for the exploration and exploitation task. To understand how the choice of representation can improve the efficiency of reinforcement learning (RL), we study representation selection for a class of low-rank Markov Decision Processes (MDPs) where the transition kernel can be represented in a bilinear form. We propose an efficient algorithm, called \algname, for representation learning in both online and offline RL. Specifically, we show that the online version of \algname, called \algname-UCB, always performs no worse than the state-of-the-art algorithm without representation selection, and achieves a strictly better constant regret if the representation function class has a "coverage" property over the entire state-action space. For the offline counterpart, \algname-LCB, we show that the algorithm can find the optimal policy if the representation class can cover the state-action space and achieves gap-dependent sample complexity. This is the first result with constant sample complexity for representation learning in offline RL.
\end{abstract}

\section{Introduction}
Reinforcement Learning (RL) has achieved impressive results in game-playing \citep{mnih2013playing}, robotics \citep{kober2013reinforcement}, and many other tasks. However, most current RL tasks are challenging due to large state-action spaces that make traditional tabular methods intractable. Instead, function approximation methods can be applied to tackle this challenge. In this scheme, the state-action pairs are compressed to provide some compact \emph{representations} that leverage the underlying structure in the MDP, and therefore allow the algorithm to generalize to unseen states.

In modern approaches, deep neural networks are often used as feature extractors to generate these representations. Since different feature extractors powered by different pretrained neural networks can be used, multiple valid representations are generated to encode the same state-action pair. However, how to select the \emph{best} representation for different scenarios is not well addressed in the literature. Nonetheless, this task is crucial in many applications such as robotics, where a robot is usually equipped with different types of sensors working through different physical phenomena \citep{de2018integrating}, like accelerometers, magnetic sensors, or laser sensors. These sensors estimate the current state of the robot and provide a representation of the current state as the output. However, the accuracy and robustness of these sensors vary in different states. Thus an intelligent system should utilize the most accurate and robust sensor in different states to achieve the best performance.

For online reinforcement learning, existing works on representation learning \citep{jiang2017contextual, agarwal2020flambe, modi2021model, uehara2021representation, sun2019model, du2021bilinear} often assumed that the transition dynamic can be represented as a linear function of an unknown representation, and they proposed algorithms to learn a single representation with provable sample complexity guarantees. They do not consider the possibility of using different representations for different scenarios (i.e., state-action pairs). On the other hand, for offline reinforcement learning, representation learning is much less studied, with only a few notable exceptions \citep{uehara2021representation, zhang2022making}. Nevertheless, neither of these works considers selecting different representations for different scenarios.

Based on the above motivation, we are interested in the following research question:

\begin{center}
\emph{Can selecting a good representation improve sample efficiency in (online and offline) RL?} \
\end{center}

In this paper, we answer the above question affirmatively for a class of low-rank Markov Decision Processes (MDPs) named bilinear MDP \citep{yang2020reinforcement}, where the transition kernel $\PP(s' | s, a)$ can be written as a bilinear form of a known feature map $\bphi(s, a)$, unknown matrix $\Mb^*$, and known feature map $\bpsi(s')$. Our goal is to select the best representation $\bphi(s,a)$ from a finite representation class $\Phi$ for different $(s,a)$ such that the resulting RL algorithm outperforms that using a single representation for all state-action pairs. For both online and offline reinforcement learning, we propose an algorithm called \algname, which can select the best representation in a representation function class in different scenarios. The key idea behind the representation selection is to choose the representation which gives the smallest optimistic Q-value function\footnote{For offline RL, our algorithm chooses the representation which gives the largest pessimistic Q-value function.}. Our contributions are summarized as follows:

\begin{itemize}[leftmargin=*,nosep]
\item In the context of online reinforcement learning, we propose a novel algorithm named \algname-UCB, which capitalizes on the benefits of representation selection. Our results show that \algname-UCB performs as well as the state-of-the-art algorithms that do not select representations, and attains a strictly superior regret bound when the representation function class has good coverage for all state-action pairs under the optimal policy.
\item For offline reinforcement learning, we introduce \algname-LCB as a counterpart to \algname-UCB for the online setting. We demonstrate that \algname-LCB is capable of identifying the optimal policy with gap-dependent sample complexity of the offline data. Furthermore, when the representation function class satisfies certain coverage assumptions under the behavior policy, our algorithm enjoys a constant sample complexity, which represents a novel contribution to this line of research.
\item To validate the effectiveness of representation selection and the superiority of our algorithms, we conduct empirical studies on various MDPs with different representation functions. Our experimental results demonstrate that both \algname-UCB and \algname-LCB outperform any single representation function in the respective settings, thus confirming the power of representation selection and the advantages of our proposed algorithms.
\end{itemize}

\noindent\textbf{Notation.} Scalars and constants are denoted by lower and upper case letters, respectively. Vectors are denoted by lower case boldface letters $\xb$, and matrices by upper case boldface letters $\Ab$. We denote by $[k]$ the set $\{1, 2, \cdots, k\}$ for positive integers $k$. For two non-negative sequence $\{a_n\}, \{b_n\}$, $a_n = \cO(b_n)$ means that there exists a positive constant $C$ such that $a_n \le Cb_n$, and we use $\tilde \cO(\cdot)$ to hide the $\log$ factor in $\cO(\cdot)$ except for the episode number $k$. We denote by $\|\cdot\|_2$ the Euclidean norm of vectors and the spectral norm of matrices and by $\|\cdot\|_{\mathrm F}$ the Frobenius norm of a matrix. We denote the Loewner ordering between two symmetric matrices as $\Ab \succeq \Bb$ if $\Ab - \Bb \succeq \zero$. For a vector $\xb \in \RR^d$, we denote by $\xb_{[i]}$ the $i$-th element of $\xb$, for a matrix $\Ab \in \RR^{d \times d}$, we denote by $\Ab_{[ii]}$ the $i$-th diagonal element. For any symmetric matrix $\Ab \in \RR^{d \times d}$ and vector $\xb \in \RR^d$, we denote %$\|\xb\|_{\Ab}^2 = \xb^\top \Ab \xb$ and 
$\|\xb\|_{\Ab} = \sqrt{\xb^\top \Ab \xb}$. We define $\Ib$ as the identity matrix. We denote the image space of a matrix $\Ab$ as $\mathrm{Im}(\Ab)$, and a vector is in the image space $\xb \in \mathrm{Im}(\Ab)$ if there exists a vector $\yb$ such that $\xb = \Ab \yb$.

\section{Related Works}
In this section, we discuss related works on representation learning and selection in both online and offline RL. Additional related works are discussed in Appendix~1.

Learning good representations in reinforcement learning has a long history. One of the earliest methods for aggregating different states and generating a compressed representation for those states is state aggregation \citep{michael1995reinforcement, dean1997model, ravindran2002model, abel2016near}. In deep RL, deep neural networks have been used to learn good representations in different settings \citep{diuk2008object, stooke2020decoupling, yang2020offline}. Several theoretical works \citep{du2019provably1, misra2020kinematic, foster2020instance} have studied the Block MDP where the dynamics are governed by a discrete latent state space and proposed algorithms based on decoding the latent state space from the observations. \citet{Du2020Is} showed that having a good approximate representation for the Q-function, transition kernel, or optimal Q-function is not sufficient for efficient learning, and can still have an exponential sample complexity unless the quality of the approximation is above a certain threshold. In the linear function approximation setting, several representation learning algorithms have been proposed. For example, \citet{jiang2017contextual} proposed a model-free algorithm called OLIVE, which can learn the correct representation from a representation function class (in the realizable setting). \citet{modi2021model} improved the OLIVE algorithm by proposing the MOFFLE algorithm, which is computationally efficient. On the other hand, \citet{agarwal2020flambe} proposed a model-based algorithm, FLAMBE, which can find the correct representation from the representation function class. \citet{uehara2021representation} improved FLAMBE by combining the maximum likelihood estimator and optimistic estimation (resp. pessimistic estimation) for representation learning in online RL (resp. offline RL). Some recent works \citep{qiu2022contrastive, zhang2022making} have used contrastive learning instead of the maximum likelihood estimator in \citet{agarwal2020flambe,uehara2021representation} to obtain more practical algorithms.

All the aforementioned works focus on learning the "correct" representation, which can well approximate the underlying transition kernel. In contrast, we pursue a different objective, which is to select a good representation adaptively for different state-action pairs from a class of correct representations, which can potentially lead to better performance. To achieve this objective, \citet{papini2021leveraging} proposed an algorithm, LEADER, which leverages good representations in linear contextual bandits. Independent of our work, \citet{papinireinforcement} extended the representation selection in linear contextual bandits to linear MDPs \citep{jin2020provably}. The differences between their work and ours are as follows. First, they considered the linear MDP setting, which yields a linear dependence on the size of the representation class (i.e., $|\Phi|$) in their regret, while we studied a special linear MDP called bilinear MDP \citep{yang2020reinforcement}, which enjoys a logarithmic dependency on $|\Phi|$ (i.e., $\log(|\Phi|)$) in our regret. Second, we also consider representation learning in offline RL, which, to our knowledge, has not been considered before in the literature. Compared to previous results on representation learning in offline RL \citep{zhang2022making, uehara2021representation}, we provide the first gap-dependent sample complexity

\section{Preliminaries}
We consider time-inhomogeneous episodic Markov Decision Processes (MDP), denoted by $\cM(\cS, \cA, H, \{r_h\}_{h=1}^H, \{\PP_h\}_{h=1}^H)$. Here, $\cS$ is the state space, $\cA$ is the finite action space, $H$ is the length of each episode, $r_h: \cS \times \cA \mapsto [0, 1]$ is the reward function at step $h$, and $\PP_h(s'| s, a)$ denotes the probability for state $s$ to transition to state $s'$ with action $a$ at step $h$. We further assume that the initial state $s_1$ is randomly sampled from a distribution $\mu$.

Given the MDP, we consider a deterministic policy $\pi = \{\pi_h\}_{h=1}^H$ as a sequence of functions where $\pi_h: \cS \mapsto \cA$ maps a state $s$ to an action $a$. For each state-action pair $(s, a) \in \cS \times \cA$ at time-step $h$, given the policy $\pi$, we denote the Q-function and value function as follows:
\begin{align*}
Q_h^\pi(s, a) &= r_h(s, a) + \EE\left[\sum_{h' = h + 1}^H r_{h'}(s_{h'},\pi_{h'}(s_{h'}))\right], \\
V_h^\pi(s) &= Q_h^\pi(s, \pi_h(s)),
\end{align*}
where $s_h = s, a_h = a$ and for all $h' \in [h, H]$, the distribution of $s_{h' + 1}$ is given by $\PP_{h'}(s_{h'} | s, a)$. Both $Q_h^\pi(s, a)$ and $V_h^\pi(s)$ are bounded in $[0, H]$ by definition. We further define the optimal value function as $V^h(s) := \sup_{\pi} V^\pi_h(s)$ and the optimal Q-function as $Q^h(s, a) := \sup_{\pi} Q^\pi_h(s, a)$. The optimal policy is denoted by $\pi_h^(s):= \argmax_\pi V^\pi_h(s)$, and we assume the optimal policy function $\pi^*$ is unique.

For simplicity, we define $[\PP_h V](s, a) = \EE_{s' \sim \PP_h(s' | s, a)} V(s')$ for any function $V: \cS \mapsto \RR$. With this notation, we have the following Bellman equation, as well as the Bellman optimality equation:
\begin{align}
Q_h^\pi(s, a) &= r_h(s, a) + [\PP_h V_{h+1}^\pi](s, a), \notag \\
Q_h^*(s, a) &= r_h(s, a) + [\PP_h V_{h+1}^*](s, a),\label{eq:bellman}
\end{align}
where $V_{H+1}^*$ and $V_{H+1}^\pi$ are set to be zero for any state $s$ and policy $\pi$.

We will focus on learning the structure of the MDP in an online manner. The algorithm is designed to run for $K$ episodes, where for each episode $k \in [K]$, the first step is to determine a policy $\pi^k = \{\pi_h^k\}_{h=1}^H$ based on the knowledge collected from the environment. The agent then follows the policy and the dynamics of the MDP. Specifically, at each step $h \in [H]$, the agent observes the state $s_h^k$, selects an action $a_h^k$ using the policy $\pi_h^k$, transitions to the next state $s_{h+1}^k$ generated by the MDP, and receives the reward $r_{h+1}^k$.

We define the cumulative regret for the first $K$ episodes as $\R{K} = \sum_{k=1}^K V_1^*(s_1^k) - V_1^{\pi^k}(s_1^k)$, where $V_1^*(s_1^k)$ is the optimal value of the initial state in episode $k$ and $V_1^{\pi^k}(s_1^k)$ is the value of the initial state in episode $k$ under the policy $\pi^k$.

The aim of this paper is to establish a problem-dependent regret bound. To achieve this goal, we require the assumption of a strictly positive minimal sub-optimality gap~\citep{simchowitz2019non, yang2021q, he2020logarithmic}. This assumption ensures that the difference between the value of the optimal policy and the value of any other policy is not too small, which is essential for proving the regret bound.

\begin{assumption}\label{asm:gap}
We have $\gap{\min}>0$, where
\begin{align}
     \gap{h}(s, a) &:= V_h^*(s) - Q_h^*(s, a), \notag \\
     \gap{\min} &:= \inf_{h, s, a}\big\{\gap{h}(s, a): \gap{h}(s, a) \neq 0\big\}. \label{def:gap}
\end{align}
\end{assumption}
We consider the bilinear MDPs in~\citet{yang2020reinforcement}, where the probability transition kernel is a bi-linear function of the feature vectors.
\begin{definition}[Bilinear MDPs, \citealt{yang2020reinforcement}]\label{asm:lin}
    For each state-action-state triple $(s, a, s') \in \cS \times \cA \times \cS$,  vectors $\bphi(s, a) \in \RR^d, \bpsi(s') \in \RR^{d'}$ are known as the feature vectors. There exists an unknown matrix $\Mb^*_h \in \RR^{d \times d'}$ for all $h \in [H]$ such that $\PP_h( s' | s, a) = \bphi^\top(s, a)\Mb^*_h\bpsi(s')$. We denote $\Kb_{\bpsi} = \sum_{s \in \cS}\bpsi(s)\bpsi^\top(s)$ which is assumed to be invertible. Let $\bPsi = (\bpsi(s_1), \bpsi(s_2), \cdots, \bpsi(s_{|\cS|}))^\top \in \RR^{|\cS| \times d'}$ be the matrix of all $\bpsi$ features. We assume that for all $h \in [H]$, 
    $\|\Mb_h^*\|_F^2 \le C_{\Mb}d$, for all $(s, a) \in \cS \times \cA$, $\|\bphi(s, a)\|_2^2 \le C_{\bphi}d$, and for all $\vb \in \RR^{|\cS|}$, $\|\bPsi^\top \vb\|_{2} \le C_{\bpsi}\|\vb\|_{\infty}$ and $\|\bPsi\Kb_{\bpsi}^{-1}\|_{2, \infty} \le C'_{\bpsi}$, where $C_{\Mb}, C_{\bphi}, C_{\bpsi}$ and $C_{\bpsi}'$ are all positive constants. 
\end{definition}

In this work, we focus on the bilinear MDP, which is a specific case of the low-rank MDP or linear MDP. In the low-rank MDP framework \citep{yang2019sample,jin2020provably}, the transition kernel is assumed to be a bilinear function of the state-action feature vector and an unknown measure $\btheta_h(s')$ of dimension $d$, i.e., $\PP_h(s'| s, a) = \la \bphi(s, a), \btheta_h(s') \ra$. In our approach, we model $\btheta_h(s')$ as a product of an unknown matrix $\Mb_h^*$ and a feature vector $\bpsi(s')$. In contrast, in the linear MDP, the reward function is assumed to be a linear function of the state-action feature vector $\bphi(s, a)$, whereas we assume it is known to simplify the presentation. As noted by \citet{yang2020reinforcement}, we can replace this assumption with a linear function of the representation $\bphi(s, a)$ and add an optimistic reward function estimation step similar to LinUCB \citep{chu2011contextual} without significantly altering the analysis.

Given the linear function representation of the MDP, we aim to learn a good representation $\bphi(s, a)$ for different state-action pairs in the representation function class $\Phi$, in both online and offline settings. To this end, we introduce the definition of an \emph{admissible} representation function class.

\begin{definition}[Admissible Function Class]\label{def:admissible}
A representation function class $\Phi$ is admissible if every $\bphi \in \Phi$ satisfies Definition~\ref{asm:lin} with a different dimension $d_{\bphi}$, a different parameter $\Mb_{h, \bphi}^*$, a different constant $C_{\bphi}$, and the same context $\bpsi(s')$. In other words, for any representation function $\bphi \in \Phi$, the same transition kernel can be represented as $\PP_h(s' | s, a) = \bphi^\top(s, a)\Mb_{h, \bphi}^*\bpsi(s')$.
\end{definition}

\begin{remark}
Definition~\ref{def:admissible} suggests that the same transition kernel can be represented in different ways, which is quite common in practice. For example, one can always represent a bilinear MDP with finite state and action spaces by a tabular MDP, which can be further represented by another bilinear MDP with $d_{\bphi} = |\cS| \times |\cA|$. However, different representations $\bphi$ may have different learning complexities. For instance, the linear representations with a lower dimension $d_{\bphi}$ are easier to learn than the tabular representation with $d_{\bphi} = |\cS| \times |\cA|$. Thus, our goal is to select a good representation from the admissible function class for different state-action pairs.
\end{remark}

\begin{remark}
    In the rest of our paper, we assume that the functions $\bphi \in \bPhi$ is given to the algorithm. In real-world applications, however, such a function class can be chosen as hand-crafted features or pre-trained neural networks.
\end{remark}
\begin{remark}\label{rm:1}
Although one can also consider learning the representation function tuple $(\bphi, \bpsi) \in \Phi \times \Psi$ simultaneously, there is no difference compared with assuming the $\bpsi$ function is fixed and known in terms of the algorithm and the analysis. This is because the $Q$-function is a linear function of $\bphi$ (see~\eqref{eq:q}), and the confidence radius of the estimated $Q$-function only depends on $\bphi$ (see~\eqref{eq:estq}). Therefore, we only select the representation of $\bphi$ instead of both $\bphi$ and $\bpsi$ without loss of generality.
\end{remark}

\section{Representation Selection for Online RL}
\subsection{\algname-UCB Algorithm}
We present the \emph{Representation seLection for EXploration and EXploitation} with upper confidence bound (\algname-UCB) algorithm for selecting a good representation from a finite representation function class $\Phi$ for different state-action pairs. The algorithm, shown in Algorithm~\ref{alg:main}, maintains a different model parameter estimate for each individual representation $\bphi \in \Phi$. Under Definition~\ref{asm:lin}, we have the following property for each representation $\bphi$:
\begin{align}
&\left[ \PP_h \bpsi(\cdot)^\top\Kb_{\bpsi}^{-1}\right ](s, a) = \sum_{s' \in \cS} \PP_h(s' | s, a) \bpsi^\top(s')\Kb_{\bpsi}^{-1} \notag\\
&\quad= \sum_{s' \in \cS} \bphi^\top(s, a)\Mb_{h, \bphi}^*\bpsi(s') \bpsi^\top(s')\Kb_{\bpsi}^{-1}\notag \\\
&\quad= \bphi^\top(s, a)\Mb_{h, \bphi}^*,\label{hhelp}
\end{align}
where the last equality uses the fact that $\Kb_{\bpsi} = \sum_{s' \in \cS} \bpsi(s')\bpsi^\top(s')$. Equation \eqref{hhelp} suggests that we can build $\Mb_{h, \bphi}^k$, the estimate of $\Mb^*_{h, \bphi}$, as the solution to the ridge regression problem analytically, given the sampled triples $\{s_h^j, a_h^j, s_{h+1}^j\}_{j=1}^{k-1}$ in Line~\ref{ln:1} of Algorithm~\ref{alg:main}.

\begin{remark}
    The computation of $\Kb_{\bpsi}$ requires only one pass of the state space since it does not depend on the round $k$ or the representation $\bphi$. Thus, it is not computationally expensive and would not be a bottleneck in the algorithm's computational complexity. Additionally, when the state space is infinite, $\Kb_{\bpsi}$ can be efficiently approximated using Monte Carlo integration techniques, as demonstrated in prior works such as \citet{zhou2020provably} and \citet{yang2020reinforcement}.
\end{remark}

With the estimate $\Mb_{h, \bphi}^k$, Algorithm~\ref{alg:main} recursively estimates the $Q$-function starting from $Q_{H+1}^k = 0$. The $Q$-function at step $h$ can be deduced as $Q_{h, \bphi}^k(s, a) = r(s, a) + [\PP_h V_{h+1}^k](s, a)$, where $V_{h+1}^k$ is the estimate of the value function at step $h+1$. Using the Bellman equation~\eqref{eq:bellman}, we can further write $Q_{h, \bphi}^k(s, a)$ as
\begin{align}
Q_{h, \bphi}^k(s, a) &= r(s, a) + \sum_{s' \in \cS} \bphi^\top(s, a) \Mb_{h, \bphi}^k\bpsi(s')V_{h+1}^k(s') \label{eq:q}.
\end{align}
To construct an optimistic estimation of the $Q$-function, we follow the approach proposed by \citet{yang2020reinforcement} and add an optimism bonus term to the right-hand side of \eqref{eq:q}. The optimism bonus is defined as $C_{\bpsi}H\sqrt{\beta_{k, \bphi}}\|\bphi(s, a)\|{(\Ub_{h, \bphi}^k)^{-1}}$, where $C_{\bpsi}$ and $\beta_{k, \bphi}$ are user-defined hyperparameters, and $\Ub_{h, \bphi}^k$ is the covariance matrix calculated in Line~\ref{ln:2}. This results in the following optimistic estimation of the $Q$-function:
\begin{align}
Q_{h, \bphi}^k(s, a) &= r(s, a) + \sum_{s' \in \cS} \bphi^\top(s, a) \Mb_{h, \bphi}^k\bpsi(s')V_{h+1}^k(s') \notag \\
&\quad+ C_{\bpsi}H\sqrt{\beta_{k, \bphi}}\|\bphi(s, a)\|_{(\Ub_{h, \bphi}^k)^{-1}}.\label{eq:estq}
\end{align}

By following the standard analysis for optimistic estimation \citep{abbasi2011improved}, it can be shown that $Q_{h, \bphi}^k(s, a)$ is an upper confidence bound for $Q_h^{\pi^k}(s, a)$ for each representation $\bphi \in \Phi$, according to~\eqref{eq:estq}. In other words, $Q_{h, \phi}^k(s, a) \ge Q_h^*(s, a)$. In Line~\ref{ln:4}, the algorithm selects the representation with the smallest optimistic estimation, which should be considered as the tightest optimistic estimation given the current covariance matrix $\Ub_{h, \bphi}^k$. Alternatively, this step can be interpreted as selecting the representation with the minimal uncertainty, which is measured by $\|\bphi\|_{(\Ub_{h, \bphi}^k)^{-1}}$. This approach is intuitive and ensures that the algorithm chooses the representation that provides the best possible estimate of the value function for the current state-action pair.

As a result, Line~\ref{ln:4} selects different representations $\bphi$ for different state-action pairs implicitly by minimizing the optimistic $Q$-function, which results in a tighter optimistic estimation of the $Q$-function. The algorithm then executes the greedy policy and obtains the optimistic value function defined in Line~\ref{ln:3}.

Our algorithm offers a distinct advantage over \citet{yang2020reinforcement} in that it enables the selection of different representations for different state-action pairs. This is in contrast to representation learning, which seeks a \emph{universal} representation that works well for \emph{all} state-action pairs. For instance, our algorithm can adaptively choose a representation that yields accurate value function estimates for certain state-action pairs, even if its performance is suboptimal for others. By doing so, our algorithm outperforms \citet{yang2020reinforcement} which relies on a single representation for all state-action pairs. This demonstrates the benefits of representation selection in online RL.

\begin{algorithm}[t!]
\caption{Online Representation seLection for EXploration and EXploitation with Upper Confidence Bound (ReLEX-UCB)}\label{alg:main}
\begin{algorithmic}[1]
    \STATE Initialize $Q_{H+1, \bphi}^k(s, a) = 0$ for all $(s, a, k, \bphi)$
    \FOR{episodes $k=1,\ldots,K$}
    \STATE Received the initial state $s_1^k$. 
    \FOR{step $h=H,\ldots,1$}
    \FOR{representation $\bphi \in \Phi$}
    \STATE $\Mb_{h, \bphi}^k=\argmin_{\Mb} \big(\sum_{j=1}^{k-1} \|\bpsi^\top(s_{h+1}^k)\Kb_{\bpsi}^{-1} - \bphi^\top(s_h^k, a_h^k)\Mb\|_2^2 + \|\Mb\|_F^2\big)$ \label{ln:1}
    \STATE $\Ub_{h, \bphi}^k = \Ib + \sum_{j=1}^{k-1}\bphi(s_h^k, a_h^k)\bphi^\top(s_h^k, a_h^k)$\label{ln:2}
    \STATE Calculate $Q_{h, \bphi}^k(s, a)$ as~\eqref{eq:estq}
    \ENDFOR
    \STATE Set $Q_h^k(s, a) = \min_{\bphi \in \Phi} \{Q_{h, \bphi}^k(s, a)\}$, \label{ln:4}
    \STATE Set $V_h^k(s) = \max\{0, \min\{\max_a Q_h^k(s, a), H\}\}$ \label{ln:3}
    \ENDFOR
    \FOR{step $h=1,\ldots,H$}
    \STATE Take action $a_h^k\leftarrow \argmax_{a} Q_h^k(s_h^k,a)$
    \STATE Receive next state $s_{h+1}^k$ \label{algorithm:line5}
    \ENDFOR
\ENDFOR
\end{algorithmic}
\end{algorithm}
\subsection{Constant Regret Bounds}
We present the regret bound for \algname-UCB, which demonstrates the advantage of representation selection in a rigorous way. To do so, we require the following assumption:

\begin{assumption}\label{asm:hls}
Suppose that the representation function class $\Phi$ is admissible. For any $(s, a, h) \in \cS \times \cA \times [H]$, there exists a representation $\bphi \in \Phi$ such that $\bphi(s, a) \in \image(\bLambda_{h, \bphi})$, where
\begin{align*}
\bLambda_{h, \bphi} := \EE_{d_{\pi^*}}[\phi(s_h, \pi_h^*(s_h))\phi^\top(s_h, \pi_h^*(s_h))]
\end{align*}
with $d_{\pi^*}$ representing the state visitation distribution induced by the optimal policy $\pi^*$. We also denote $\sigma_{h, \bphi}$ as the minimal non-zero eigenvalue of $\bLambda_{h, \bphi}$ and $\sigma_{\bphi} = \min_{h \in [H]} \sigma_{h, \bphi}$.
\end{assumption}
\begin{remark}
Several related assumptions, known as \emph{diversity assumptions}, have been proposed to lower bound the minimum eigenvalue of the term $\bphi\bphi^\top$. These assumptions are discussed in detail in \citet{papini2021leveraging}. Here, we extend the assumption from the linear bandit to the reinforcement learning setting, where the state distribution at time-step $h$ is defined by the optimal policy. We note that a similar but stronger assumption, called `uniformly excited features', is made by \citet{wei2021learning} in the infinite time-horizon average reward MDP setting. There, they assume $\bLambda$ is strictly positive definite for \emph{all} possible policies $\pi$. In contrast, we only require $\bLambda$ to be strictly positive for the distribution induced by the \emph{optimal} policy, which is a weaker assumption. This implies that the states that rarely occur in the optimal policy do not significantly impact the quality of the representation.
\end{remark}

\begin{remark}
We notice that a similar assumption called UniSoft-mixing, is made in~\citet{papinireinforcement}, where they assume that for all $(s, a) \in \cS \times \cA$, there exists a $\bphi \in \Phi$ such that $\bphi(s, a) \in \text{span }\{\bphi(s, \pi^*(s)): d_{\pi^*}(s) > 0\}$. The difference between our assumption and their assumption is that they filter out the states which are \emph{almost surely} never visited by the optimal policy $\pi^*$. In contrast, we take the expectation with respect to $d_{\pi^*}(s)$ without explicitly filtering out the never-visited states.
\end{remark}%}
Now we are ready to present the regret bound result.
\begin{theorem}\label{thm:main}
Under Assumptions~\ref{asm:gap} and~\ref{asm:hls}, set $\beta_{k, \bphi} = c(C_{\Mb} + {C'_{\bpsi}}^{2})d_{\bphi}\log(kHC_{\bphi} |\Phi| / \delta)$ in Algorithm~\ref{alg:main}, where $c$ is an absolute positive constant, then with probability at least $1 - 5\delta$, there exists a threshold
\begin{align}
    k^*= \max_{\bphi \in \Phi}\Big\{\text{poly}(d_{\bphi}, \sigma_{\bphi}^{-1}, H, \log(|\Phi| / \delta), \gap{\min}^{-1} \notag \\
    \quad, C_{\bphi}, C_{\bpsi}, C_{\Mb}, C_{\bpsi}')\Big\} \label{eq:k-star}
\end{align} 
independent from episode number $k$. The regret for the first $k$ episodes is upper bounded by 
\begin{align*}
    \R{k} &\le 2 + \min_{\bphi \in \Phi}\bigg\{ \frac{128C_{\bpsi}^2H^5d_{\bphi}^2c(C_{\Mb} + {C'_{\bpsi}}^{2})}{\gap{\min}}\\
    &\qquad \times \log\big(1 + C_{\bphi}\tilde kd_{\bphi}\big)\log\big(\tilde kHC_{\bphi} |\Phi| / \delta\big)\bigg\}\\
    &\quad + \frac{96H^4\log\big(2\tilde k (1 + \log(H \gap{\min}^{-1}))|\Phi|/ \delta\big)}{\gap{\min}}\\
    &\quad+ \frac{16H^2\log\big(\big(1 + \log\big(H{\tilde k}\big)\big){\tilde k}^2|\Phi|/ \delta\big)}3,
\end{align*}
where we denote $\tilde k := \min\{k, k^*\}$.
\end{theorem}

\begin{remark}
The regret bound exhibits a phase transition as the episode number $k$ increases. When $k \le k^*$, the regret is upper bounded by $\tilde \cO(d^2H^5\log(k)\gap{\min}^{-1})$, which is exactly the logarithmic regret bound (given by Lemma 3.3 in Appendix). However, when $k \ge k^*$, the regret bound becomes $\tilde \cO(d^2H^5\log(k^*)\gap{\min}^{-1})$. Since $k^*$ is independent of $k$ (as shown in~\eqref{eq:k-star}), the regret bound turns into a problem-dependent constant regret bound that no longer grows as the total number of episodes $k$ increases. This result aligns with our intuition: once we have a fixed, strictly positive sub-optimality gap, the regret might initially increase over the first few episodes. However, once the agent collects enough data, it can learn the environment well and will no longer incur any additional regret.
\end{remark}

\begin{remark}
If Assumption~\ref{asm:hls} does not hold, then $k^*=\infty$, and our regret bound degenerates to the gap-dependent regret bound. Similar bounds have been proved in~\citet{he2020logarithmic} for both linear MDPs and linear mixture MDPs. Our bound has the same dependency on $H$, $\gap{\min}$, and episode number $k$ as the bounds in~\citet{he2020logarithmic}. However, in terms of $d$, our dependency is $\cO(d^2)$, while the dependency is $\cO(d^3)$ for linear MDPs in the LSVI-UCB algorithm~\citep{jin2020provably}. This difference arises from estimating the MDP parameter $\Mb_{h, \bphi}^*$, which is similar to that in the UCRL-VTR algorithm~\citep{ayoub2020model} for learning linear mixture MDPs. Furthermore, since our regret bound minimizes over all $\bphi \in \Phi$, the performance of \algname-UCB is always competitive with the best one using any single representation $\bphi$ in that function class, ignoring the logarithmic terms.
\end{remark}

\begin{remark}\label{rm:logphi}
We note that our regret bound includes an additional $\log(|\Phi|)$ factor, which reflects the cost of representation selection to guarantee that all $|\Phi|$ regressions can be learned well by the union bound. This term is caused by the worst-case scenario and may be eliminated in practice by considering the average-case scenario instead. By doing so, we can potentially reduce the impact of the $\log(|\Phi|)$ factor on the regret bound. Additionally, it's worth noting that this dependency on $|\Phi|$ is better than the one in the regret bound of \citet{papinireinforcement}, which has a $|\Phi|$ factor. The reason for the better dependency in our result is that the bilinear MDP structure we consider is simpler than the linear MDP structure considered in~\citet{papinireinforcement}. When applying our algorithm to linear MDPs, we still need a $|\Phi|$ factor to cover the value function class, which degenerates to the result in~\citet{papinireinforcement}. Furthermore, the $\log(|\Phi|)$ dependency allows us to extend our result to some infinite representation function classes with bounded statistical complexity~\citep{agarwal2020flambe}.
\end{remark}

\begin{remark}
When $|\Phi| = 1$, i.e., there is only one representation function, Assumption~\ref{asm:hls} provides a criterion for a `good representation' and such a `good representation' can improve the problem-dependent regret bound from $\cO(\log(k))$~\citep{he2020logarithmic} to a constant regret bound. 
\end{remark}

\section{Representation Selection for Offline RL}
\subsection{\algname-LCB Algorithm}
We present an offline version of $\algname$ that selects a good representation based on the offline data generated from a behavior policy. In this version, the algorithm estimates the parameter and its covariance matrix for each representation function $\bphi$ in Lines~\ref{ln:1-offline} and \ref{ln:2-offline} in Algorithm~\ref{alg:main-offline}, using the offline data $\cD_h$ for the $h$-th step, which consists of the triplet $(s, a, s')$ as the state, action, and next-state, then the estimated $\Mb$ can be therefore written by
\begin{align}
\Mb_{h, \bphi} & =\argmin_{\Mb} \|\Mb\|_F^2 \notag \\
&\quad + \sum_{(s, a, s') \in \cD_h} \|\bpsi^\top(s')\Kb_{\bpsi}^{-1} - \bphi^\top(s, a)\Mb\|_2^2  \label{eq:offline-M}
\end{align}
The algorithm then provides a pessimistic estimation of the $Q$-function, following a similar method as~\eqref{eq:estq} in Lines~\ref{ln:4-offline} and~\ref{ln:5-offline}, which is widely used in offline reinforcement learning to provide a robust estimation for later planning. In detail, the estimated Q function is subtracted by a confidence radius $\Gamma$ defined by
\begin{align}
\Gamma_{h, \bphi}(s, a) = C_{\bpsi}H\sqrt{\beta_{\bphi}\bphi^\top(s, a)\Ub_{h, \bphi}^{-1}\bphi(s, a)} \label{eq:offline-C}
\end{align}
and thus the estimated $Q$-function can be written as
\begin{align}
Q_{h, \bphi}(s, a) &= r(s, a) + \sum_{s' \in \cS} \bphi^\top(s, a) \Mb_{h, \bphi}\bpsi(s')V_{h+1}(s') \notag \\
&\quad - \Gamma_{h, \bphi}(s, a) \label{eq:offline-Q}.
\end{align}

Unlike the online version, where a smaller estimation of $Q$ is preferred, the offline version adopts a pessimistic estimation ($Q_{h, \bphi} \le Q_h^*$), where a larger estimation is considered more accurate. Therefore, in Line~\ref{ln:3-offline}, the algorithm selects the maximum $Q$-function over all representation functions $\bphi$, and in Line~\ref{algorithm:line5-offline}, it takes the greedy policy based on the selected $Q$-function from the offline training. Similar to the online version, $\algname$-LCB selects different representation functions for different state-action pairs instead of a single representation for the entire environment, thereby leveraging the advantage of different representation functions to provide a good estimation for the underlying MDP.
\begin{algorithm}[t!]
\caption{Offline Representation seLection for EXploration and EXploitation Lower Confidence Bound (ReLEX-LCB)}\label{alg:main-offline}
\begin{algorithmic}[1]
\STATE \texttt{// offline training}
\FOR{$(h, \bphi) \in [H] \times \Phi$}
\STATE Calculate $\Mb_{h, \bphi}$ as of~\eqref{eq:offline-M} \label{ln:1-offline}
\STATE Calculate $\Ub_{h, \bphi} = \Ib + \sum_{(s, a, s') \in \cD_h}\bphi(s, a)\bphi^\top(s, a)$\label{ln:2-offline}
\ENDFOR
\STATE \texttt{// offline planning}
\STATE Initialize $Q_{H+1, \bphi}(s, a) = 0$ for all $(s, a, \bphi)$
\FOR {$h = H, H - 1, \cdots, 1$}
\STATE Calculate $\Gamma_{h, \bphi}(s, a)$ as of~\eqref{eq:offline-C}\label{ln:4-offline}
\STATE Calculate $Q_{h, \bphi}(s, a)$ as of~\eqref{eq:offline-Q}\label{ln:5-offline}
\STATE Set $Q_h(s, a) = \max_{\bphi \in \Phi} \{Q_{h, \bphi}(s, a)\}$
\label{ln:3-offline}
\STATE Set $V_h(s) = \max\{0, \min\{\max_a Q_h(s, a), H\}\}$
\STATE Set $\pi_h(s, a) = \argmax_a Q_h(s, a)$ \label{algorithm:line5-offline}
\ENDFOR
\ENSURE Policy $\pi = \{\pi_h\}_{h=1}^H$
\end{algorithmic}
\end{algorithm}

\subsection{Gap-dependent Sample Complexity}
In this section, we provide the sample complexity of Algorithm~\ref{alg:main-offline}. Similarly to its online counterpart, we start with a coverage assumption for offline RL, which suggests that the representation function class $\Phi$ can provide a good representation for all possible state-action pairs in the offline training data.

\begin{assumption}\label{asm:offline}
Suppose the representation function class $\Phi$ is admissible, and for any $(s, a, h) \in \cS \times \cA \times [H]$, there exists a representation function $\bphi \in \Phi$ such that
\begin{align*}
\bphi(s, a) \in \image(\tilde \bLambda_{h, \bphi}),\ \tilde \bLambda_{h, \bphi} := \EE_{d_h^{\hat \pi}}[\bphi(s, a)\bphi(s, a)^\top],
\end{align*}
where $d_h^{\hat \pi}$ is the state-action visitation distribution in the offline dataset on step $h$ induced by some behavior policy $\hat \pi$ in the underlying MDP for the offline data. We denote the minimal non-zero eigenvalue of $\tilde \bLambda_{h, \bphi}$ as $\tilde \sigma_{h, \bphi}$.
\end{assumption}

\begin{remark}
Similar assumptions have been made in the offline RL literature~\citep{wang2020statistical,jin2021pessimism, min2021variance, uehara2021representation, yin2022nearoptimal}, which require that the offline dataset can provide good coverage of the entire state-action space. Notably, thanks to representation selection, we only require that the representations in the function class $\Phi$ can together cover the state-action space, rather than every single representation covering the state-action space perfectly. This relaxes existing assumptions by allowing every single representation to not provide perfect coverage. For example, it is possible to define two representations $\{\bphi_1, \bphi_2\}$ such that each representation does not satisfy Assumption~\ref{asm:offline}, but the representation function class $\Phi = {\bphi_1, \bphi_2}$ satisfies. For more details about this example, please refer to Appendix~2.2, or Appendix G in \citet{papinireinforcement}.
\end{remark}

We also need the following assumption, which is standard in the literature.
\begin{assumption}\label{asm:iid}
    The trajectories in the offline dataset are i.i.d. sampled, i.e., different trajectories are generated by the same behavior policy $\hat \pi$ independently.
\end{assumption}

Now we are ready to present the sample complexity result.
\begin{theorem}\label{thm:offline}
Set $\beta_{\bphi} = Cd_{\bphi}\log(2KH|\Phi|/\delta)$ where $C$ is an absolute positive constant, then with probability at least $1 - \delta$, then under Assumptions~\ref{asm:offline} and~\ref{asm:iid}, the sub-optimality of the policy $\pi$ output by Algorithm~\ref{alg:main-offline} could be bounded by
% \begin{smaller}
\begin{align}
    &V_h^*(s) - V_h^\pi(s) \le 2C_{\bpsi}H\notag \\
    &\qquad\times\sum_{h' = h}^H\EE_{\pi^*}\Big[\min_{\bphi \in \Phi}\big\{\sqrt{\beta_{\bphi}}\|\bphi(s, a)\|_{\Ub_{h', \bphi}^{-1}}\big\}\big | s_h = s\Big].\label{eq:main}
\end{align}
% \end{smaller}
Furthermore, under Assumptions~\ref{asm:gap}, if the size of the offline dataset is greater than
% \begin{small}
\begin{align*}
    K &> \max_{\bphi \in \Phi, h \in [H]} \left\{ \frac{32C_{\bphi}^2d_{\bphi}^2\log(Hd_{\bphi}|\Phi|/\delta)}{\tilde \sigma^{2}_{h, \bphi}}\right.\\
    &\qquad\qquad\quad\left. \times\left(1 + \frac{C_{\bpsi}^2H^4\beta_{\bphi}C_{\bphi}\tilde \sigma_{h, \bphi}}{4\gap{\min}^2C_{\bphi}^2d_{\bphi}\log(Hd_{\bphi}|\Phi|/\delta)}\right)\right\}.
\end{align*}
% \end{small}
then Algorithm~\ref{alg:main-offline} is guaranteed to output the optimal policy $\pi = \pi^*$.
\end{theorem}

\begin{remark}
Our error bound in \eqref{eq:main} contains the $\min$ operator, which suggests that our result should be no worse than using any single representation, compared with the offline RL algorithm using a single representation~\citep{jin2021pessimism, yin2022nearoptimal}. 
\end{remark}

\begin{remark}
    The bound of $\sqrt{\beta_{\bphi}}\|\bphi(s, a)\|_{\Ub_{h', \bphi}^{-1}}$ cannot decrease to $0$ without other further assumptions. \cite{jin2021pessimism,yin2022nearoptimal} require a `uniform coverage' assumption to make the sub-optimality decrease at a $1/\sqrt{K}$ rate. This `uniform coverage' suggests that the covariance matrix under the behavior policy can cover the entire state-action space. In sharp contrast, according to Assumption~\ref{asm:offline}, our results only require the representations in the function class to together cover the state-action space, even if any single representation cannot.
\end{remark}

\begin{remark}
Our `gap-dependent sample complexity' is also aligned with the gap-dependent sample complexity for offline RL in the tabular setting under the condition $(P, \text{gap}_{\min})$ in \citet{wang2022gap}. In their setting, $P$ stands for a uniform optimal policy coverage coefficient in the tabular MDP, which is analogous to our $\tilde{\sigma}_{h, \bphi}^{-1}$ in the linear function approximation setting. Our result has the same inverse dependence on $\text{gap}_{\min}$.
\end{remark}

\section{Experiments}
\begin{table}
\caption{Cumulative regret ($\text{mean} \pm \text{dev.}$) after 5M episodes for \algname-UCB v.s. UC-MatrixRL and $\epsilon$-greedy using a single representation} \label{tab:reg}
\centering
\begin{tabular}{cc}
\toprule
Alg. + Rep.  & Cumulative regret \\ 
\midrule
UC-MatrixRL + $\bphi$ (oracle) & $2534.9 \pm 26.6$ \\ 
\midrule
UC-MatrixRL + $\bphi^{(1)}$ & $11459.5 \pm 225.7$ \\
UC-MatrixRL + $\bphi^{(2)}$ & $13838.5 \pm 266.2$ \\
$\epsilon$-greedy + $\bphi$ & $15305.9 \pm 245.7$ \\ 
$\epsilon$-greedy + $\bphi^{(1)}$ & $15745.8 \pm 408.0$ \\
$\epsilon$-greedy + $\bphi^{(2)}$ & $15652.9 \pm 471.2$ \\
\algname-UCB + $\{\bphi^{(1)}, \bphi^{(2)}\}$ & $\boldsymbol{6765.0 \pm 146.6}$\\
\bottomrule
\end{tabular}
\end{table}
\subsection{Online RL}
To showcase the efficacy of representation selection by \algname-UCB, we conduct the following experiments on an environment with $|\cS| = 20, |\cA| = 3$, $H = 10$, and $d = d' = 5$. We generate the feature functions $\bphi: \cS \times \cA \mapsto \RR^d$ and $\bpsi: \cS \mapsto \RR^{d'}$ such that for all $h \in [H]$, there exists a matrix $\Mb_h \in \RR^{d \times d}$ where $\PP_h(s' | s, a) = \bphi(s, a)^\top \Mb_h \bpsi(s')$. The generated $\bphi$ satisfies Assumption~\ref{asm:hls}. We set the reward function such that $r_H(s, a) \sim \text{Bernoulli}(0.5)$ and $r_h(s, a) = 0$ for all $h < H$, forcing the algorithm to learn the transition kernel in order to achieve good performance.

Furthermore, we generate two additional representations $\bphi^{(1)}$ and $\bphi^{(2)}$ such that neither $\bphi^{(1)}$ nor $\bphi^{(2)}$ satisfies Assumption~\ref{asm:hls}, but their union $\Phi = {\bphi^{(1)}, \bphi^{(2)}}$ does. Appendix~2.1 contains a detailed definition of these representations.

We evaluated the performance of \algname-UCB using the feature map class $\Phi = \{\bphi^{(1)}, \bphi^{(2)}\}$ with episode $K = 5,000,000$. We also reported the performance of UC-MatrixRL~\citep{yang2020reinforcement} and $\epsilon$-greedy using the feature map $\bphi$, $\bphi^{(1)}$, and $\bphi^{(2)}$ separately.

We repeated the experiment on the same environment eight times and reported the mean and standard deviation of the cumulative regret in Table~\ref{tab:reg}. Our experiment results showed that \algname-UCB outperformed both $\epsilon$-greedy and UC-MatrixRL using $\bphi^{(1)}$ or $\bphi^{(2)}$, which verifies the effectiveness of representation selection. More results, including the figure of cumulative regret, are deferred to Appendix 2.4.1.

% Figure~\ref{fig:my_label} plots the cumulative regret with respect to the episode number, with the standard deviation indicated by the shadows. We observed that the cumulative regrets for both UC-MatrixRL using $\bphi$ and \algname-UCB grew very slowly after the first one million episodes. As a comparison, UC-MatrixRL using $\bphi^{(1)}$ or $\bphi^{(2)}$ had a sub-linear regret growth instead of near-constant regret. As for the $\epsilon$-greedy algorithm, although the greedy policy can learn very fast at the beginning, it eventually had a much higher cumulative regret since it could not explore the environment well.

% \begin{figure}[t]
%     \centering
%     \includegraphics[width=.7\columnwidth ]{aistats/z-crop.pdf}
%     \caption{Cumulative regret over 5M episodes for \algname-UCB v.s. UC-MatrixRL and $\epsilon$-greedy using a single representation.}
%     \label{fig:my_label}
% \end{figure}

\subsection{Offline RL}
\begin{table}
\caption{Relative sub-optimality of \algname-LCB over 500K episodes} \label{tab:2}
\begin{center}
\begin{tabular}{cc}
\toprule
\makecell{Representation}  & \makecell{Final sub-optimality \\$(\text{mean} \pm \text{dev.}) \times 10^{-3}$} \\ 
\midrule
\makecell{$\bphi$ (oracle)} & $1.288 \pm 0.807$ \\ 
$\bphi^{(1)}$ & $3.424 \pm 1.455$ \\
$\bphi^{(2)}$ & $3.336 \pm 1,624$ \\
$\{\bphi^{(1)}, \bphi^{(2)}\}$ & $\boldsymbol{1.292 \pm 0.806}$\\
\bottomrule
\end{tabular}
\end{center}
\end{table}
In this subsection, we present experiments to demonstrate the performance of \algname-LCB. We use a setup similar to the online RL setting with one oracle representation $\bphi$ satisfying Assumption~\ref{asm:offline} and two representations $\bphi^{(1)}$ and $\bphi^{(2)}$. Neither of $\bphi^{(1)}$ nor $\bphi^{(2)}$ satisfies Assumption~\ref{asm:offline}, but the union of these two representations satisfies the assumption. We collect $K = 500K$ episodes of offline trajectories using a fixed randomly-generated behavior policy and evaluate the sub-optimality of Algorithm~\ref{alg:main-offline} using different sizes of offline training data. The rest of the parameter settings are the same as in the online RL setting.

We report the performance of Algorithm~\ref{alg:main-offline} using (1) the oracle representation $\bphi$, (2) the representation function class $\{\bphi_1, \bphi_2\}$, (3) $\bphi_1$, and (4) $\bphi_2$, respectively. We use the relative sub-optimality over the initial policy, i.e., $(V_1^*(s) - V_1^{\pi_k}(s)) / (V_1^*(s) - V_1^{\pi_1}(s))$ as a performance measure. We repeat the experiment 32 times and report the mean and standard deviation of the relative sub-optimality in Table~\ref{tab:2}.

We observe that by selecting over two imperfect representations, \algname-LCB can match the performance of the oracle algorithm using a single perfect representation, even if using the two representations separately leads to a larger ($\sim 2.5\times$) sub-optimality on the same offline data. More results, including the figures comparing the sub-optimality over different algorithms, are deferred to Appendix 2.4.1.

% \begin{figure}[t]
% \centering
% \includegraphics[width=.85\columnwidth]{NeurIPS2022/offline.pdf}
% \caption{Relative sub-optimality of \algname-LCB after 500K offline episodes}
% \label{fig:2}
% \end{figure}

\section{Conclusion and Future Work}
\label{sec:conclusion}

In this paper, we have explored representation selection for reinforcement learning by focusing on a special class~\citep{yang2020reinforcement} of low-rank MDPs~\citep{yang2019sample,jin2020provably}. Our proposed \algname algorithm has demonstrated the ability to improve performance in both online and offline RL settings. The promising theoretical and empirical results suggest that there is potential in combining our work with FLAMBE~\citep{agarwal2020flambe} or MOFFLE~\citep{modi2021model}. By integrating our approach with these methods that select the \emph{correct} representations, we can further select the \emph{good} representation from a class of \emph{correct} representations. This may help in designing more practical, theory-backed representation learning algorithms for reinforcement learning.
\begin{acknowledgements} % will be removed in pdf for initial submission,
						 % (without ‘accepted’ option in \documentclass)
                         % so you can already fill it to test with the
                         % ‘accepted’ class option
    We thank the anonymous reviewers for their helpful comments. WZ, JH, DZ and QG are supported in part by the National Science Foundation CAREER Award 1906169 and research fund from UCLA-Amazon Science Hub. The views and conclusions contained in this paper are those of the authors and should not be interpreted as representing any funding agencies.
\end{acknowledgements}
% \clearpage
% References
\bibliography{zhang_466}
\clearpage
\onecolumn
\appendix 

\section{Additional Related Work}\label{app:related}
\noindent\textbf{Reinforcement Learning with Linear Function Approximation.}
A large body of literature regarding learning MDP with linear function approximation has emerged recently. Those works can be roughly divided by their assumptions on MDPs: The first one is called Linear MDP~\citep{yang2019sample,jin2020provably}, where the representation function is built on the state action pair $\bphi(s, a)$. Under this assumption,~\citet{jin2020provably} proposed the LSVI-UCB algorithm achieving $\tilde \cO(\sqrt{d^3H^3T})$ problem independent regret bound and $\tilde \cO(d^3H^5\gap{\min}^{-1}\log(T))$ problem dependent regret bound due to ~\citet{he2020logarithmic}. Here $\gap{\min}$ is the minimal sub-optimality gap, $d$ is the dimension and $H$ is the time-horizon. Several similar MDP assumptions are studied in the literature: for instance,~\citet{jiang2017contextual} studied a larger class of MDPs with low Bellman rank and proposed an algorithm with polynomial sample complexity. Low inherent Bellman error assumption is proposed by~\citet{zanette2020learning} and allows a better $\cO(dH\sqrt{T})$ regret by considering a global planning oracle.~\citet{yang2020reinforcement} considered the bilinear structure of the MDP kernel as a special case of the Linear MDP, and achieved an $\tilde \cO(H^2d\sqrt{T})$ problem-independent regret bound. The second linear function approximation assumption is called Linear Mixture MDP~\citep{modi2020sample,ayoub2020model,zhou2020provably} where the transition kernel of MDP is a linear function $\bphi(s, a, s')$ of the `state-action-next state' triplet. Under this setting,~\citet{jia2020model,ayoub2020model} proposed UCRL-VTR achieving $\cO(d\sqrt{H^3T})$ problem independent bound for episodic MDP, while~\citet{he2020logarithmic} showed an $\tilde \cO(d^2H^5\gap{\min}^{-1}\log^3(T))$ problem dependent regret bound for the same algorithm.~\citet{zhou2020provably} studied infinite horizon MDP with discounted reward setting and proposed UCLK algorithm to achieve $\tilde \cO(d\sqrt{T}(1 - \gamma)^2)$ regret. Most recently,~\citet{zhou2020nearly} proposed nearly minimax optimal algorithms for learning Linear Mixture MDPs in both finite and infinite horizon settings.

However, these works all assume a single representation and do not depend on the quality of the representation as long as it can well approximate the value function. Thus, what a good representation is and what improvement this good representation can bring is still an open question.

\noindent\textbf{Offline Reinforcement Learning with Function Approximation} There is a series of works focusing on the offline reinforcement learning with linear function approximation. \citet{jin2021pessimism} introduce the pessimism to offline reinforcement learning and establish a data-dependent upper bound on the sub-optimality for general MDP. They also provide a close-formed data-dependent bound for linear MDPs. Following that, \citet{xie2021bellman} introduces the notion of Bellman's consistent pessimism for general function approximation. There is also a brunch of work leveraging the variance information in offline RL~\citep{min2021variance, yin2021near, yin2022nearoptimal}. Other follow-up works include the partial coverage~\citep{uehara2021pessimistic} in general function approximation and the statistical barriers for offline RL~\citep{foster2021offline}.

\noindent\textbf{Model Selection and Representation Learning in Contextual Bandits.} Since contextual bandits can be viewed as a special case of MDPs, our work is also related to some previous works on model selection in contextual bandits. The first line of work runs a multi-armed bandit at a high level while each arm corresponds to a low level contextual bandit algorithm. Following this line,~\citet{odalric2011adaptive} used a variant of EXP4~\citep{auer2002nonstochastic} as the master algorithm while the EXP3 or UCB algorithm~\citep{auer2002nonstochastic} serves as the base algorithm.
This result is improved by CORRAL~\citep{agarwal2017corralling}, which uses the online mirror descent framework and modifies the base algorithm to be compatible with the master.~\citet{pacchiano2020model} introduced a generic smoothing wrapper that can be directly applied to the base algorithms without modification.

\citet{abbasi2020regret} proposed a regret balancing strategy and showed that given the regret bound for the optimal base algorithm as an input, their algorithm can achieve a regret that is close to the regret of the optimal base algorithm. Following that,~\citet{pacchiano2020regret} relaxed the requirement in~\citet{abbasi2020regret} by knowing each base algorithm comes with a candidate regret bound that may or may not hold during all rounds. Despite this progress, how to get the optimal regret guarantee for the general contextual learning problem remains an open question~\citep{foster2020open}. Besides those general model selection algorithms, recent works are focusing on representation learning under several different structures, thus different representations can be used at different rounds in the algorithm.~\citet{foster2019model} studied model selection by considering a sequence of feature maps with increasing dimensions where the losses are linear in one of these feature maps. They proposed an algorithm that adaptively learns the optimal feature map, whose regret is independent of the maximum dimension.~\citet{chatterji2020osom} studied the hidden simple multi-armed bandit structure where the rewards are independent of the contextual information. 
\citet{ghosh2021problem} considered a nested linear contextual bandit problem where the algorithm treats the norm bound or dimension of the weight vector in the linear model as the complexity of the problem and adaptively finds the true complexity for the given dataset. 
\section{Experiment Details}\label{app:exp}
\subsection{Online reinforcement learning}
Here we describe how to generate the representation functions and the MDP. We denote the $d$-dimensional half normal distribution by $|\xb| \sim \cH(\Ib_d)$ if $\xb \sim \cN(\zero_d, \Ib_d)$. 
Considering the following representation sample from the half-normal distribution for all $(s, a, s') \in \cS \times \cA \times \cS$:
\begin{align*}
    \tilde \bphi(s, a) \sim \cH(\Ib_d), \tilde \bpsi(s') \sim \cH(\Ib_d).
\end{align*}

Then we define $\bpsi$ by $\bpsi(s') = \tilde \bpsi(s') / \max_{s \in \cS} \|\tilde \bpsi(s)\|_2$. It is obvious that the Euclidean norm of $\bpsi(s')$ is bounded by 1 for all $s' \in \cS$.

It is easy to tell that each element in $\tilde \bphi$ and $\bpsi$ is non-negative thus we can build the transition kernel as
\begin{align*}
    \PP_h(s' | s, a) = \frac{\tilde \bphi(s, a)^\top \bpsi(s')}{\sum_{s' \in \cS}\tilde \bphi(s, a)^\top \bpsi(s')}.
\end{align*}

Next, for any step $h \in [H]$, given any non-singular matrix $\Mb_h \in \RR^{d\times d}$, we define representation function $\bphi_h(\cdot, \cdot)$ as 
\begin{align}
    \bphi_h(s, a) =  \frac{\Mb_h^{-1}\tilde \bphi(s, a)}{\sum_{s' \in \cS}\tilde \bphi(s, a)^\top \bpsi(s')}. \label{eq:oracle}
\end{align}

Furthermore, we select the matrix $\Mb_h$ such that for all state-action pair $(s, a) \in \cS \times \cA$, $\|\bphi_h(s, a)\|_2 \le 1$. This procedure could always be done since we can multiply different scalars to the generated matrix $\Mb_h$ to control the norm of $\bphi_h$.

Therefore, we can verify that for any $(s, a, s', h) \in \cS \times \cA \times \cS \times [H]$, $\PP_h(s' | s, a) = \bphi^\top_h(s, a) \Mb_h \bpsi(s')$ thus it satisfies Assumption~\ref{asm:hls}. To emphasize the difficulties of learning the transition kernel $\PP$.

It is easy to verify that under the current representation $\bphi$, with high probability, $\bLambda_{h, \bphi} \succeq \zero $ since the representation $\bphi$ is sampled from the half-normal distribution. Therefore, Assumption~\ref{asm:hls} is satisfied.

We will next provide two other representations $\{\bphi^{(1)}, \bphi^{(2)}\}$ for the same transition kernel $\PP_h(\cdot | \cdot, \cdot)$. Neither of these single representation satisfies Assumption~\ref{asm:hls} but the combination of these two will satisfy that assumption.

Since the transition kernel $\PP_h(\cdot | \cdot, \cdot)$ and reward function $r(\cdot, \cdot)$ have already been determined, by Bellman optimality equation~\eqref{eq:bellman}, we can get the optimal action action $\pi^*_h(s)$ for all step $h \in [H]$ and state $s \in \cS$. Since $|\cA| = 3$ and $\pi^*_h(s) \in \cA$, we can compose the sub-optimal set by $\cA \setminus \{\pi^*_h(s)\} := \{a_h(s), a'_h(s)\}$. Then we define the two representation functions as $\bphi^{(1)}, \bphi^{(2)} \in \cS \times \cA \mapsto \RR^{2d}$ using the following rule: %there are two suboptimal actions, where we randomly label them as $a_h(s), a'_h(s)$. Therefore, we define two representation set 
\begin{align*}
    \begin{cases}
    \bphi^{(1)}_h(s, \pi_h^*(s)) = \left(\bphi_h^\top(s, \pi_h^*(s)), \zero_d^\top\right)^\top\\
    \bphi^{(1)}_h(s, a_h(s)) = \left(\bphi_h^\top(s, a_h(s)), \zero_d^\top\right)^\top\\
    \bphi^{(1)}_h(s, a'_h(s)) = \left(\zero_d^\top, \bphi^\top_h(s, a'_h(s))\right)^\top
    \end{cases},
    \begin{cases}
    \bphi^{(2)}_h(s, \pi_h^*(s)) = \left(\bphi_h^\top(s, \pi_h^*(s)), \zero_d^\top\right)^\top\\
    \bphi^{(2)}_h(s, a_h(s)) = \left(\zero_d^\top, \bphi_h^\top(s, a_h(s)) \right)^\top\\
    \bphi^{(2)}_h(s, a'_h(s)) = \left(\bphi^\top_h(s, a'_h(s)), \zero_d^\top\right)^\top
    \end{cases}
\end{align*}
By constructing the new kernel matrix $\tilde \Mb_h = (\Mb_h^\top, \Mb_h^\top)^\top \in \RR^{2d \times d}$, we can verify that both $\bphi^{(1)}$ and $\bphi^{(2)}$ satisfy Assumption~\ref{asm:lin} with dimension $2d$. i.e. for all $(s, a, s', h) \in \cS \times \cA \times \cS \times [H]$
\begin{align}
   \PP_h(s' | s, a) = \bphi^{(1)\top}_h\tilde \Mb_h\bpsi(s') =  \bphi^{(2)\top}_h\tilde \Mb_h\bpsi(s').
\end{align}

From intuition, these two representation functions $\bphi^{(1)}$ and $\bphi^{(2)}$ may come from two different sensors measuring the same environment. It is obvious that since for both $\bphi^{(1)}$ and $\bphi^{(2)}$, there are at least 1/3 of the whole state-action space is not covered by $\bLambda_h$, i.e. $\bphi_h^{(1)}(s, a'_h(s)) \notin  \image \bLambda_{h, \bphi^{(1)}}$ and $\bphi_h^{(2)}(s, a_h(s)) \notin  \image \bLambda_{h, \bphi^{(2)}}$. However, since $a'_h(s') \neq a_h(s)$ by definition, we can verify that the representation set $\Phi = \{\bphi^{(1)}, \bphi^{(2)}\}$ satisfies Assumption~\ref{asm:hls}.

\subsection{Offline reinforcement learning}\label{app:exp-offline}

Here we provide a design for representation functions such that each single representation does not satisfy Assumption~\ref{asm:offline} but the whole representation function set satisfies.

First, the oracle representation which satisfies Assumption~\ref{asm:offline} is generated same as~\eqref{eq:oracle} with $d = 5$. The underlying MDP is generated same with the online settings with $|\cS| = 20, |\cA| = 3$. Then we consider an arbitrary behavior policy $\hat \pi$ which is used to generate the offline training data. Since $|\cA| = 3$, for any $s \in \cS, h \in [H]$, there exists three state-action pairs as $(s, \hat \pi_h(s)), (s, a_h(s))$ and $(s, a'_h(s))$. Considering the representation set $\bphi^{(1)}(\cdot, \cdot) \in \RR^{2d}$ and $\bphi^{(2)}(\cdot, \cdot) \in \RR^{2d}$ which is defined as 
\begin{align*}
\begin{cases}
    \bphi^{(1)}_h(s, \hat \pi_h(s)) &= \left(\bphi_h^\top(s, \hat \pi_h(s))^\top, \zero_d^\top\right)^\top\\
    \bphi^{(1)}_h(s, a_h(s)) &= \left(\bphi_h^\top(s, a_h(s))^\top, \zero_d^\top\right)^\top\\
    \bphi^{(1)}_h(s, a'_h(s)) &= \left(\zero_d^\top, \bphi_h^\top(s, a'_h(s))^\top\right)^\top
\end{cases},
\begin{cases}
    \bphi^{(2)}_h(s, \hat \pi_h(s)) &= \left(\bphi_h^\top(s, \hat \pi_h(s))^\top, \zero_d^\top\right)^\top\\
    \bphi^{(2)}_h(s, a_h(s)) &= \left(\zero_d^\top, \bphi_h^\top(s, a_h(s))^\top \right)^\top\\
    \bphi^{(2)}_h(s, a'_h(s)) &= \left(\bphi_h^\top(s, a'_h(s))^\top, \zero_d^\top\right)^\top
\end{cases}.
\end{align*}
It is obvious that by using behavior policy $\hat \pi$, both $\EE_{d_h^{\hat \pi}}(\bphi^{(1)}\bphi^{(1)\top})$ and $\EE_{d_h^{\hat \pi}}(\bphi^{(2)}\bphi^{(2)\top})$ would enjoy the format of $\begin{pmatrix} \bLambda & \zero \\ \zero & \zero \end{pmatrix}$. Therefore, for $\bphi^{(1)}$, it would be easy to verify that $\bphi_h^{(1)}$, $\bphi^{(1)}(s, a'_h(s))$ is not in $\image(\EE_{d_h^{\hat \pi}}(\bphi^{(1)}\bphi^{(1)\top}))$ and $(\bphi^{(2)}(s, a_h(s))$ is not in $\image(\EE_{d_h^{\hat \pi}}(\bphi^{(2)}\bphi^{(2)\top}))$. However, it is also easy to show that the union of $\{\bphi^{(1)}, \bphi^{(2)}\}$ satisfy assumption~\ref{asm:offline}.

\subsection{Additional Configuration} \label{app:more}

\noindent\textbf{Parameter Tuning.} For both of the offline and online algorithm, we aggregate the parameter $C_{\bpsi}H\sqrt{\beta_{k, \bphi}}$ as a single hyper-parameter $C$ for tuning. We do a grid search for $C = \{1, 3, 10, 30, 100\}$ report the best performance over these values.

\subsection{Additional Results}
\subsubsection{Online RL}
Figure~\ref{fig:my_label} plots the cumulative regret with respect to the episode number, with the standard deviation indicated by the shadows. We observed that the cumulative regrets for both UC-MatrixRL using $\bphi$ and \algname-UCB grew very slowly after the first one million episodes. As a comparison, UC-MatrixRL using $\bphi^{(1)}$ or $\bphi^{(2)}$ had a sub-linear regret growth instead of near-constant regret. As for the $\epsilon$-greedy algorithm, although the greedy policy can learn very fast at the beginning, it eventually had a much higher cumulative regret since it could not explore the environment well.

\begin{figure}
    \centering
    \parbox{0.45\textwidth}{
    \includegraphics[height=6cm]{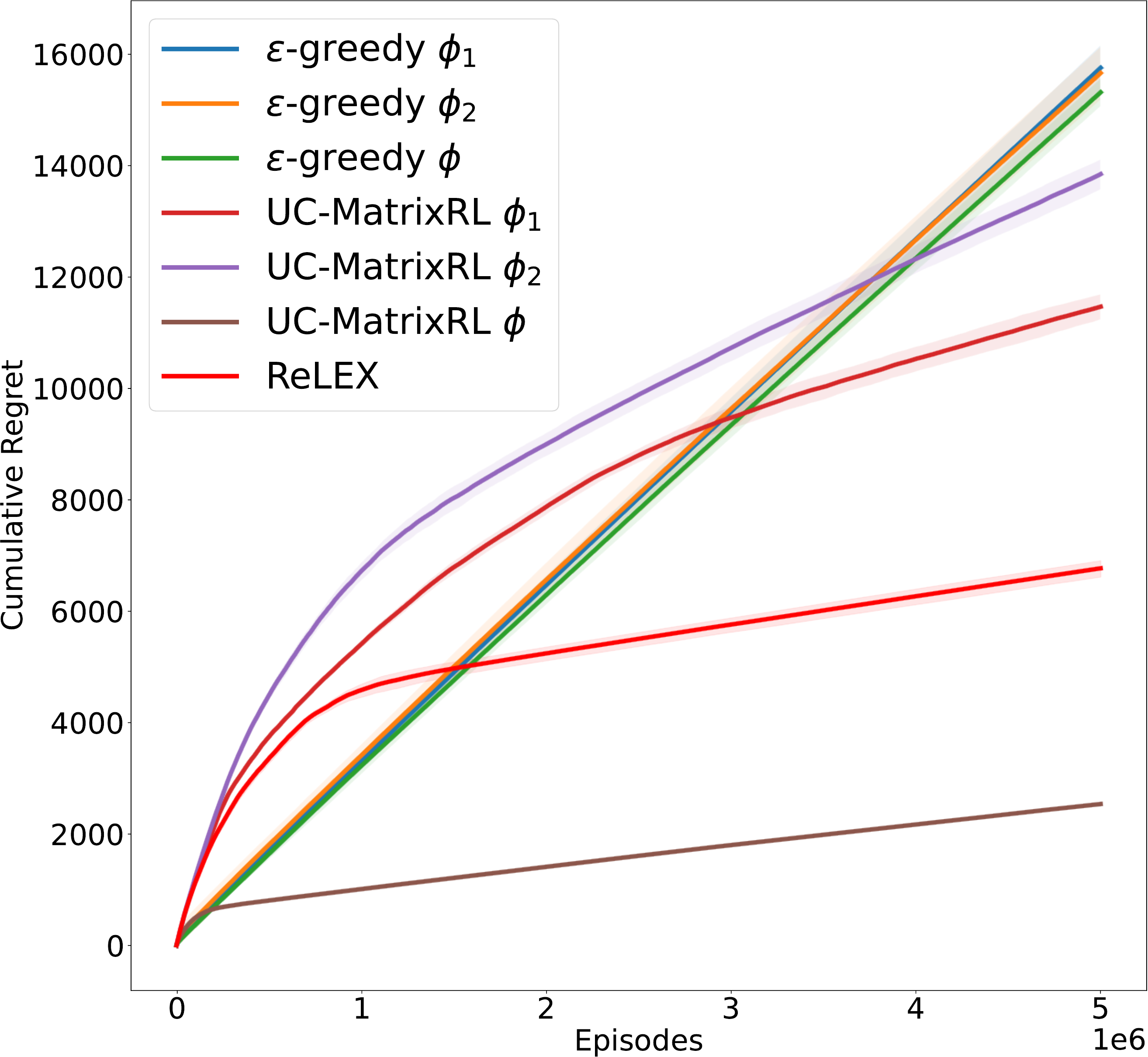}
    \caption{Cumulative regret over 5M episodes for \algname-UCB v.s. UC-MatrixRL and $\epsilon$-greedy using a single representation.}
    \label{fig:my_label}}%
  \qquad
  \begin{minipage}{0.45\textwidth}%
    \includegraphics[height=6cm]{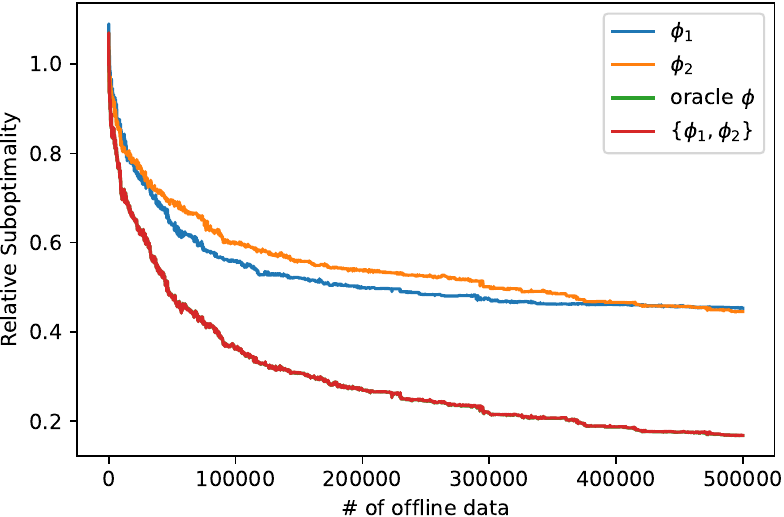}
    \caption{Relative sub-optimality of \algname-LCB after 500K offline episodes}%
    \label{fig:2}%
  \end{minipage}%
\end{figure}

\subsubsection{Offline RL}
From Figure~\ref{fig:2} we observe that by selecting over two imperfect representations, \algname-LCB can match the performance of the oracle algorithm using a single perfect representation, even if using the two representations separately leads to a larger ($\sim 2.5\times$) sub-optimality on the same offline data.

\subsubsection{Ablation studies}

We conduct additional experiments on the following different settings in the online setting as the ablation study of our algorithm.

\begin{enumerate}[label=(\Alph*)]
    \item The original setting with same data generation, $S = 20, A = 3, d = d' = 5, H = 10$
    \item The original setting with larger state space $S = 40$, other parameters are not changed
    \item The original setting with larger action $A = 5$, other parameters are not changed
    \item The original setting with larger action $S = 40, A = 5$, other parameters are not changed
    \item The original setting with $|\Phi| = 3, A = 4$, other parameters are not changed
\end{enumerate}

Besides the cumulative regret, we also report the average reward achieved in the last 1000 episodes, which can be considered a "more intuitive performance metric" suggested by Reviewer 4Uhq.

Regarding the data generation, configuration C to E enjoys the same method of generating the data, i.e., arranging the context of sub-optimal actions into other dimensions. We will add the details of generating these data during the revision.

Due to the time limit of the authors' response, we do not repeat the experiments multiple times and we cut experiments E and F with episode $K = 500,000$ instead of the original $K = 5,000,000$ in the paper. We also skipped the $\epsilon$-greedy version for configurations E for the sake of time. 

The performance table are presented From Table~\ref{tab:abl1} to Table~\ref{tab:abl5}

\noindent\textbf{Computing Resources}
For both offline and online algorithm, we conduct our experiments on an AWS c5-12xlarge CPU instance with a 48-core Intel\textsuperscript{\textregistered} Xeon\textsuperscript{\textregistered} Scalable Processors (Cascade Lake).
\begin{table}[htbp]
    \centering
    \begin{tabular}{c|cc}
    \toprule
    Algorithm & Last averaged reward $\uparrow$ & Cumulative regret $\downarrow$\\
    \midrule
    UC-Matrix RL $\phi$ (oracle) & 0.7782 & 2457.71\\
    ReLEX $\{\phi_1, \phi_2\}$ & \textbf{0.7780} & \textbf{6827.18}\\
    UC-Matrix RL $\phi_1$ & \textbf{0.7780} & 11458.83\\
    UC-Matrix RL $\phi_2$ & 0.7770 & 13385.61\\
    $\epsilon$-greedy $\phi$ & 0.7754 & 14906.31\\
    $\epsilon$-greedy $\phi_1$ & 0.7756 & 15042.24\\
    $\epsilon$-greedy $\phi_2$ & 0.7751 & 16470.94\\
    \bottomrule
    \end{tabular}
    \caption{The performance result of Configuration (A) (The same configuration in the paper)}
    \label{tab:abl1}
\end{table}
\begin{table}[htbp]
    % \vspace{-1em}
    \centering
    \begin{tabular}{c|cc}
    \toprule
    Algorithm & Last averaged reward $\uparrow$ & Cumulative regret $\downarrow$\\
    \midrule
    UC-Matrix RL $\phi$ (oracle) & 0.8736 & 2880.71\\
    ReLEX $\{\phi_1, \phi_2\}$ & \textbf{0.8733} & \textbf{7745.49}\\
    UC-Matrix RL $\phi_1$ & 0.8729 & 12759.89\\
    UC-Matrix RL $\phi_2$ & 0.8730 & 10411.89\\
    $\epsilon$-greedy $\phi$ & 0.8703 & 18786.77\\
    $\epsilon$-greedy $\phi_1$ & 0.8707 & 19002.81\\
    $\epsilon$-greedy $\phi_2$ & 0.8702 & 20018.97\\
    \bottomrule
    \end{tabular}
    % \vspace{-0.5em}
    \caption{The performance result of Configuration (B) ($S = 40$)}
    \label{tab:abl2}
\end{table}
\begin{table}[htbp]
    % \vspace{-1em}
    \centering
    \begin{tabular}{c|cc}
    \toprule
    Algorithm & Last averaged reward $\uparrow$ & Cumulative regret $\downarrow$\\
    \midrule
    UC-Matrix RL $\phi$ (oracle) & 0.9749 & 3085.21 \\
    ReLEX $\{\phi_1, \phi_2\}$ & \textbf{0.9748} & \textbf{8160.39} \\
    UC-Matrix RL $\phi_1$ & 0.9743 & 12946.34\\
    UC-Matrix RL $\phi_2$ & 0.9745 & 14373.86\\
    $\epsilon$-greedy $\phi$ & 0.9690 & 28423.82\\
    $\epsilon$-greedy $\phi_1$ & 0.9701 & 27479.92\\
    $\epsilon$-greedy $\phi_2$ & 0.9708 & 27778.66\\
    \bottomrule
    \end{tabular}
    % \vspace{-0.5em}
    \caption{The performance result of Configuration (C) ($A = 5$)}
    \label{tab:abl3}
\end{table}
\begin{table}[htbp]
    % \vspace{-1em}
    \centering
    \begin{tabular}{c|cc}
    \toprule
    Algorithm & Last averaged reward $\uparrow$ & Cumulative regret $\downarrow$\\
    \midrule
    UC-Matrix RL $\phi$ (oracle)& 0.9800 & 3403.65\\
    ReLEX $\{\phi_1, \phi_2\}$ & \textbf{0.9792} & 9733.84\\
    UC-Matrix RL $\phi_1$ & 0.9787 & 10000.15\\
    UC-Matrix RL $\phi_2$ & 0.9791 & \textbf{9553.96}\\
    $\epsilon$-greedy $\phi$ & 0.9763 & 23301.53\\
    $\epsilon$-greedy $\phi_1$ & 0.9759 & 23392.78\\
    $\epsilon$-greedy $\phi_2$ & 0.9758 & 23218.40\\
    \bottomrule
    \end{tabular}
    % \vspace{-0.5em}
    \caption{The performance result of Configuration (D) ($S=40, A=5$) }
    \label{tab:abl4}
\end{table}%\vspace{-10em}
\begin{table}[htbp]
    \centering
    \begin{tabular}{c|cc}
    \toprule
    Algorithm & Last averaged reward $\uparrow$ & Cumulative regret $\downarrow$\\
    \midrule
    UC-Matrix RL $\phi$ (oracle)& 0.9081 & 1141.63\\
    UC-Matrix RL $\phi_1$ & 0.9034 & 4512.82\\
    UC-Matrix RL $\phi_2$ & 0.9008 & 4965.15\\
    UC-Matrix RL $\phi_3$ & 0.9022 & 4507.18\\
    ReLEX $\{\phi_1, \phi_2\}$ & 0.9051 & 2865.86\\
    ReLEX $\{\phi_1, \phi_3\}$ & 0.9057 & 2606.99\\
    ReLEX $\{\phi_2, \phi_3\}$ & 0.90648 & 2702.94\\
    ReLEX $\{\phi_1, \phi_2, \phi_3\}$ & \textbf{0.9080} & \textbf{2093.32}\\
    \bottomrule
    \end{tabular}
    % \vspace{-0.5em}
    \caption{The performance result of Configuration (E) ($S=20, A=4, |\Phi| = 3$)}
    \label{tab:abl5}
\end{table}
\section{Proof of Theorem~\ref{thm:main}}\label{sec:app-online-1}
In this section we will give the key technical lemmas and the proof sketch for Theorem~\ref{thm:main}.

First, we need to define a ``good event'' which happens with high probability, that the estimation $\Mb_h^k$ is close to the target $\Mb_h^*$. This definition was originally introduced in~\citet{yang2020reinforcement}.

\begin{lemma}[Lemma 15,~\citet{yang2020reinforcement}]\label{lm:good}
Define the following event as $\cE_{\bphi}^k$,
\begin{align*}
    \Big\{\tr \left[(\Mb_{h, \bphi}^j - \Mb^*_{h, \bphi})^\top \Ub_{h, \bphi}^j (\Mb_{h, \bphi}^j - \Mb^*_{h, \bphi})\right] \le \beta_{j, \bphi}, \forall j \le k, \forall h \in [H] \Big\} =: \cE_{\bphi}^k,
\end{align*}
With $\beta_{k, \bphi} = c(C_{\Mb} + {C'_{\bpsi}}^{2})d_{\bphi}\log(kHC_{\bphi} / \delta)$ for some absolute constant $c > 0$, we have $\Pr(\cE_{\bphi}^K) \ge 1 - \delta$ for all %$k \in [K]$ and
$\bphi \in \Phi$.
\end{lemma}
\begin{remark}
The proof of Lemma~\ref{lm:good} remains the same since the regression does not depend on the policy $\pi$. We also make the dependency of $\delta$ explicit in $\beta_{k, \bphi}$, which can be inferred from the proof of Lemma 15 in~\citet{yang2020reinforcement}.
\end{remark}

The next lemma shows a problem-dependent regret bound for the bilinear MDP in Definition~\ref{asm:lin}.

\begin{lemma}\label{lm:he}
Under Assumption~\ref{asm:gap}, setting parameter $\beta_{k, \bphi}$ as in Theorem~\ref{thm:main}. Then suppose $\cE_{\bphi}^K$ holds for all $\bphi \in \Phi$. Then with probability at least $1 - 3\delta$, the regret for the very first $k \in [K]$ episodes is controlled by
\begin{align}
    \R{k} &\le \min_{\bphi \in \Phi}\bigg\{ \frac{128C_{\bpsi}^2H^5d_{\bphi}\beta_{k, \bphi}\log(1 + C_{\bphi}kd_{\bphi})}{\gap{\min}}\bigg\} + \frac{16}3H^2\log(((1 + \log(H{k})){k}^2|\Phi|/ \delta) \notag \\
    &\quad + 2 + \frac{96H^4\log(2k (1 + \log(H / \gap{\min}))|\Phi|/ \delta)}{\gap{\min}}\label{eq:hereg}
\end{align}
while the sub-optimality gap for each $h$ is controlled by
\begin{align}
    \sum_{i=1}^{k} (V_h^*(s_h^i)-Q_h^{*}(s_h^i,a_h^i)) &\le \min_{\bphi \in \Phi}\bigg\{\frac{64C_{\bpsi}^2H^4d_{\bphi}\beta_{k, \bphi}\log(1 + C_{\bphi}{k}d_{\bphi}) + 48H^3\log(2k|\Phi|(1 + \log(H / \gap{\min}) / \delta)}{\gap{\min}}\bigg\}\label{eq:hegap}.
\end{align}
\end{lemma}

It is easy to verify that when there is only one representation function in $\Phi$ (let $d=d_{\bphi}$ for simplicity), Lemma~\ref{lm:he} yields an $\cO(H^5d^2\log(k / \delta)\gap{\min}^{-1})$ problem-dependent bound. %\todoq{$\beta_{k,\bphi}$ also contains $d$?}. 
Comparing our result with~\citet{he2020logarithmic}, ours matches the problem-dependent bound for Linear Mixture MDP $\cO(H^5d^2\log(k / \delta)\gap{\min}^{-1})$ and is better than the problem-dependent bound for Linear MDP $\cO(H^5d^3\log(k / \delta)\gap{\min}^{-1})$ by a factor $d$. This improvement is due to the bilinear MDP structure in Definition~\ref{asm:lin}. Moreover, it is obvious that when $|\Phi| > 1$, Algorithm~\ref{alg:main} can achieve a regret no worse than any possible regret achieved by a single representation, up to an additive $\log(|\Phi|)$ term. Lemma~\ref{lm:he} also suggests an $\cO(H^4d^2\log(k / \delta)\gap{\min}^{-1})$ bound for the summation of the sub-optimality gap. %\todoq{$\beta_{k,\bphi}$ also contains $d$?}. 
Based on this, the next lemma shows that the ``covariance matrix'' $\Ub_{h, \bphi}^k$ is almost linearly growing with respect to $k$ under Assumption~\ref{asm:hls}.

\begin{lemma}\label{lm:lingrow}
    Under Assumptions~\ref{asm:gap} and~\ref{asm:hls}, with probability at least $1 - \delta$, we have for all $k \in [K], h \in [H], \bphi \in \Phi$,
    \begin{align*}
        \Ub_{h, \bphi}^k &\succeq (k-1)\bLambda_{h, \bphi} - \iota\Ib_{d_{\bphi}},\\
        \iota &= \frac{C_{\bphi}d_{\bphi}}{\gap{\min}}\sum_{i=1}^h\sum_{j=1}^{k-1} \gap{i}(s_i^j, a_i^j) - 1 + C_{\bphi}d_{\bphi}\sqrt{32H(k-1)\log(d_{\bphi}|\Phi|Hk(k+1) / \delta)}.
    \end{align*}
\end{lemma}

Compared with Lemma 9 in~\citet{papini2021leveraging} which shows a similar result for linear contextual bandits, the proof of Lemma~\ref{lm:lingrow} is more challenging: The distribution of $s_h$ is induced by the optimal policy $\pi^*$ in Assumption~\ref{asm:hls} but we can only use the estimated policy $\pi^k$ to sample $s_h$. As a result, the sub-optimality and the randomness for the steps before $h$ ($i<h$ ) will all contribute to this distribution mismatch. Therefore, our result contains an additional summation over $h$ to account for this effect. %\todoq{this comment is not very clear}

Finally, equipped with Lemma~\ref{lm:lingrow}, we can provide a constant threshold $\tau$ such that if the episode number $k$ goes beyond $\tau$, the sub-optimality gap is bounded by $\cO(\sqrt{1 / k})$.

\begin{lemma}\label{lm:gap}
    Under Assumptions~\ref{asm:gap} and~\ref{asm:hls}, assuming the conditions in Lemmas~\ref{lm:he} and~\ref{lm:lingrow} hold and $\cE_{\bphi}^K$ holds for all $\bphi \in \Phi$, then there exists a threshold 
    \begin{align*}
    \tau&= \text{poly}(d_{\bphi}, \sigma_{\bphi}^{-1}, H, \log(|\Phi| / \delta), \gap{\min}^{-1}, C_{\bphi}, C_{\bpsi}, C_{\Mb}, C_{\bpsi}')
    \end{align*}
    such that for all $\tau \le k \le K$, for all $h \in [H], s \in \cS$ we have 
    \begin{align*}
        \gap{h}(s, \pi^k_h(s)) \le 2C_{\bpsi}H^2\max_{\bphi \in \Phi}\Big\{d_{\bphi}\sqrt{2C_{\bphi}\beta_{k, \bphi} / (\sigma_{\bphi}k)}\Big\}.
    \end{align*}
\end{lemma}

Lemma~\ref{lm:gap} suggests that when the episode number $k$ exceeds $\tau$, policy $\pi^k$ will contribute a sub-optimality up to $\cO(\sqrt{1 / k})$. Thus there exists a threshold $k^*$ such that when $k \ge k^*$, the policy $\pi^k$ will not contribute any sub-optimality at any step $h$ given the minimal sub-optimality gap $\gap{\min}$. With that, it suffices to provide the proof for Theorem~\ref{thm:main}.

\begin{proof}[Proof of Theorem~\ref{thm:main}]
We pick $\beta_{k, \bphi} = c(C_{\Mb} + {C'_{\bpsi}}^{2})d_{\bphi}\log(kHC_{\bphi}|\Phi| / \delta)$ to make sure with probability at least $1 - \delta$, event $\cE_{\bphi}^K$ holds for all $\bphi \in \Phi$.
By the definition of the sub-optimality gap, we have $\gap{h}(s, a) \ge \gap{\min}(s, a)$ as long as $\gap{h}(s, a) \neq 0$. If
\begin{align}
    k > \max
    \left\{\frac{8C_{\bpsi}^2H^4}{\gap{\min}^2}\max_{\bphi \in \Phi}\left\{\frac{C_{\bphi}d^2_{\bphi}\beta_{k, \bphi}}{\sigma_{\bphi}}\right\}, \tau\right\}\label{eq:b1}.
\end{align}
Then we have for all $\bphi \in \Phi$,
\begin{align*}
    2C_{\bpsi}H^2d_{\bphi}\sqrt{2C_{\bphi}\beta_{k, \bphi} / (\sigma_{\bphi}k)} < \gap{\min}.
\end{align*}
Thus, by Lemma~\ref{lm:gap}, we have
\begin{align*}
    \gap{h}(s, \pi^k_h(s)) &\le 2C_{\bpsi}H^2 \max_{\bphi \in \Phi}\left\{d_{\bphi}\sqrt{2C_{\bphi}\beta_{k, \bphi} / (\sigma_{\bphi}k)}\right\},
\end{align*}
and it implies $\gap{h}(s, \pi^k_h(s)) > \gap{\min}$ thus $\gap{h}(s, \pi^k_h(s)) = 0$. Since from the parameter setting, $\beta_{k, \bphi} = \cO(\log(k))$, it is easy to verify that there exists a threshold $k^* = \text{poly}(C_{\bpsi}, C_{\bpsi}', C_{\Mb}, C_{\bphi}, d_{\bphi}, H, \sigma^{-1}_{\bphi}, \gap{\min}^{-1}, \tau)$ such that all $k \ge k^*$ satisfy~\eqref{eq:b1}. Thus we conclude that $\gap{h}(s, \pi^k_h(s)) = 0$ for all $k \ge k^*$. Since the optimal policy $\pi^*$ is unique, it follows that $\pi^k = \pi^*$. 

Thus when $k \ge k^*$, the regret could be decomposed by 
\begin{align*}
    \R{k} &= \sum_{j=1}^k V_1^*(s_1^j) - V_1^{\pi^j}(s_1^j) \\
    &= \sum_{j=1}^{k^*} V_1^*(s_1^j) - V_1^{\pi^j}(s_1^j) + \sum_{j=1 + k^*}^{k} V_1^*(s_1^j) - V_1^{\pi^j}(s_1^j) \\
    &= \R{k^*} + 0,
\end{align*}
where the last equation is due to the fact that $\pi^j = \pi^*$ thus $V_1^*(s) = V_1^{\pi^j}(s)$ for all $s \in \cS$ when $j \ge k^*$. Combining this case with the case $k \le k^*$, we can conclude that $\R{k} \le \R{\min\{k, k^*\}}$. Let $\tilde k = \min\{k, k^*\}$, by Lemma~\ref{lm:he}, we have the regret is bounded by
\begin{align*}
    \R{k} &\le \min_{\bphi \in \Phi}\bigg\{ \frac{128C_{\bpsi}^2H^5d_{\bphi}^2c(C_{\Mb} + {C'_{\bpsi}}^{2}) }{\gap{\min}}\log\Big(1 + C_{\bphi}\tilde kd_{\bphi}\Big) \log\Big(\tilde kHC_{\bphi} |\Phi| / \delta\Big)\bigg\} + 2\\
    &\quad+ \frac{96H^4\log\Big(2\tilde k (1 + \log(H / \gap{\min}))|\Phi|/ \delta\Big)}{\gap{\min}} \\
    &\quad+ \frac{16}3H^2\log\Big(\Big(\Big(1 + \log(H{\tilde k}\Big)\Big){\tilde k}^2|\Phi|/ \delta\Big),
\end{align*}
%\todoq{why the third line looks in that way?}
with probability at least $1 - 5\delta$ by taking the union bound of Lemma~\ref{lm:he}, Lemma~\ref{lm:lingrow} and $\cE_{\bphi}^K$ holds.
\end{proof}

\section{Proof of Lemmas in Appendix~\ref{sec:app-online-1}}\label{app:proof}
In this section, we provide the proof of the technical lemmas in Appendix~\ref{sec:app-online-1}. 

\subsection{Filtration}
To facilitate our proof, we define the filtration list as follows
\begin{align*}
    \cF_h^k = \left\{\left\{s_i^j, a_i^j\right\}_{i=1, j=1}^{H, k - 1}, \left\{s_i^k, a_i^k\right\}_{i=1}^{h}\right\}.
\end{align*}
It is easy to verify that $s_h^k, a_h^k$ are both $\cF_{h}^k$-measurable. Also, for any function $f$ built on $\cF_{h}^k$, $f(s_{h+1}^k) - [\PP_h f](s_h^k, a_h^k)$ is $\cF_{h+1}^k$-measurable and it is also a zero-mean random variable conditioned on $\cF_{h}^k$.

Arranging the filtrations as 
\begin{align*}
    \cF = \{\cF_{1}^{1}, \cdots, \cF_{H}^{1}, \cdots, \cF_{1}^{k}, \cdots, \cF_{h}^{k}, \cdots \cF_{H}^{k}, \cdots, \cF_{1}^{K}, \cdots, \cF_{H}^{K}\},
\end{align*}
we will use $\cF$ as the filtration set for the following proof.

\subsection{Proof of Lemma~\ref{lm:he}}

To prove this lemma, we first need the following lemma showing the estimator $Q_{h,\bphi}^k$ is always optimistic.

\begin{lemma}\label{lm:optim}
Suppose the event $\cE_{\bphi}^K$ holds for all $\bphi \in \Phi$, then for any $(s, a) \in \cS \times \cA$, $Q^*_h(s, a) \le Q_{h, \bphi}^k(s, a)$.
\end{lemma}

The next lemma suggests the error between the estimated $Q$-function and the target $Q$-function at time-step $h$ can be controlled by the error at $(h+1)$-th step and the UCB bonus term.

\begin{lemma}\label{lm:gap1}
Suppose the event $\cE_{\bphi}^K$ holds, then for any $(s, a) \in \cS \times \cA, k \in [K]$ and any policy $\pi$,
\begin{align*}
    Q_h^k(s, a) - Q_h^{\pi}(s, a) \le 2C_{\bpsi} H\sqrt{\beta_{k, \bphi}\bphi^\top(s, a)(\Ub_{h, \bphi}^k)^{-1}\bphi(s, a)} + [\PP_h (V^k_{h+1} - V^\pi_{h+1})](s, a).
\end{align*}
\end{lemma}

%To prove Lemma~\ref{lm:he}, 
We also need the following lemma, which is similar to Lemma 6.2 in~\citet{he2020logarithmic}. %Then~\eqref{eq:hereg} and~\eqref{eq:hegap} can be easily get following the proof in~\citet{he2020logarithmic}.

\begin{lemma}\label{lm:he62}
For any $0 < \Delta \le H$, if the event $\cE_{\bphi}^K$ holds for all $\bphi \in \Phi$, then with probability at least $1 - \delta$, for any $k \in [K]$, 
\begin{align*}
    &\sum_{j=1}^k \ind[V_h^*(s_h^j) - Q_h^{\pi^j}(s_h^j, a_h^j) \ge \Delta] \le \frac{16C_{\bpsi}H^4d_{\bphi}\beta_{k, \bphi}\log(1 + C_{\bphi}kd_{\bphi}) + 12H^3\log(2k / \delta)}{\Delta^2}.
\end{align*}
\end{lemma}%\todoj{need $\cE_{\bphi}^K$ for all $\bphi$}

Then we need the following lemma from~\citet{he2020logarithmic} to upper-bound the regret by the summation of the sub-optimailty.

\begin{lemma}[Lemma 6.1, revised,~\citet{he2020logarithmic}]\label{lm:sumgap}
For each MDP $\cM$, with probability at least $1 - 2\delta$,  for all $k \in [K]$, we have
\begin{align*}
    \R{k} \le 2\sum_{j=1}^k\sum_{h=1}^H\gap{h}(s_h^j, a_h^j) + \frac{16H^2}3\log(((1 + \log(Hk))k^2/ \delta) + 2.
\end{align*}
\end{lemma}
\begin{remark}
Lemma~\ref{lm:sumgap} can be easily obtained from Lemma 6.1 in~\citet{he2020logarithmic}. In the original lemma, with probability at least $1 - \lceil \log HK \rceil \exp(-\tau)$,
\begin{align*}
    \R{k} \le 2\sum_{j=1}^k\sum_{h=1}^H\gap{h}(s_h^j, a_h^j) + \frac{16H^2\tau}3 + 2,
\end{align*}
which implies that with probability at least $1 - \delta$, 
\begin{align*}
    \R{k} \le 2\sum_{j=1}^k\sum_{h=1}^H\gap{h}(s_h^j, a_h^j) + \frac{16H^2\log(\lceil\log(Hk)\rceil/ \delta)}3 + 2.
\end{align*}
By relaxing $\lceil \log(Hk)\rceil$ to $\log(HK) + 1$ and replacing $\delta$ with $\delta / k^2$ for different episode number $k$, the inequality holds with probability at least $1 - \sum_{k=1}^K \delta / k^2 \ge 1 - \pi^2 \delta  / 6\ge 1 - 2\delta$ for all $k \in[K]$ by union bound.
\end{remark}

Equipped with these lemmas, we can begin our proof.
\begin{proof}[Proof of Lemma~\ref{lm:he}]
By the definition of $\gap{\min}$, for each $h \in [H], k \in [K]$, we have $V_h^*(s_h^k) - Q_h^*(s_h^k,a_h^k) = 0$ or $\gap{\min} \le V_h^*(s_h^k) - Q_h^*(s_h^k,a_h^k) \le H$. Dividing the interval $[\gap{\min}, H]$ into $N$ intervals $[2^{n-1}\gap{\min}, 2^n \gap{\min})$ where $n \in [N], N = \lceil \log(H / \gap{\min})\rceil$, then with probability at least $1 - \lceil \log(H / \gap{\min})\rceil\delta$, it holds that
\begin{align}
    \sum_{j=1}^k(V_h^*(s_h^j) - Q_h^*(s_h^j, a_h^j))
    &\le \sum_{n=1}^N\sum_{j=1}^k2^n\gap{\min}\ind[2^{n-1} \gap{\min} \le V_h^*(s_h^j) - Q^*_h(s_h^j, a_h^j) \le 2^n\gap{\min}] \notag  \\
    &\le \sum_{n=1}^N\sum_{j=1}^k2^n\gap{\min}\ind[2^{n-1} \gap{\min} \le V_h^*(s_h^j) - Q^{\pi^j}_h(s_h^j, a_h^j)] \notag \\
    &\le \sum_{n=1}^N \frac{64C_{\bpsi}^2H^4d_{\bphi}\beta_{k, \bphi}\log(1 + C_{\bphi}kd_{\bphi}) + 48H^3\log(2k / \delta)}{2^n\gap{\min}} \notag \\
    &\le \frac{64C_{\bpsi}^2H^4d_{\bphi}\beta_{k, \bphi}\log(1 + C_{\bphi}kd_{\bphi}) + 48H^3\log(2k / \delta)}{\gap{\min}}, \notag
\end{align}
where the first inequality holds by using the ``peeling technique'', which was used in local Rademacher complexity analysis~\citep{bartlett2005local}. The second inequality in~\eqref{eq:a1} is due to $Q^*_h(s, a) \ge Q^{\pi^j}_h(s, a)$ and the third inequality holds due to Lemma~\ref{lm:he62}. Finally, the fourth inequality holds due to $\sum_{n=1}^N 2^{-n} \le 1$. Substituting $\delta$ with $\delta / (1 + \log(H / \gap{\min}))$, with probability at least $1 - \delta$, we have
\begin{align}
    \sum_{j=1}^k(V_h^*(s_h^j) - Q_h^*(s_h^j, a_h^j))
&\le \frac{64C_{\bpsi}^2H^4d_{\bphi}\beta_{k, \bphi}\log(1 + C_{\bphi}kd_{\bphi}) + 48H^3\log\big(2k (1 + \log(H / \gap{\min})/ \delta\big)}{\gap{\min}} \label{eq:a1},
\end{align}
Combining~\eqref{eq:a1} with Lemma~\ref{lm:sumgap}, by taking a union bound, with probability at least $1 - 3\delta$,
\begin{align*}
    \R{k} &\le 2\sum_{j=1}^k\sum_{h=1}^H\gap{h}(s_h^k, a_h^k) + \frac{16H^2\log(\lceil HK\rceil / \delta)}{3} + 2\\
    &\le \frac{128C_{\bpsi}^2H^5d_{\bphi}\beta_{k, \bphi}\log(1 + C_{\bphi}kd_{\bphi}) + 96H^4\log(2k(1 + \log(H / \gap{\min})) / \delta)}{\gap{\min}} \\
    &\quad+ \frac{16}3H^2\log(((1 + \log(Hk))k^2/ \delta) + 2.
\end{align*}
where the first inequality holds due to Lemma~\ref{lm:sumgap} , which utilizes the definition of the sub-optimality gap. The second inequality holds due to~\eqref{eq:a1}. Substituting $\delta$ with $\delta / |\Phi|$, the claimed result~\eqref{eq:hereg} holds for all $\bphi \in \Phi$ by taking a union bound.
\end{proof}

\subsection{Proof of Lemma~\ref{lm:lingrow}}
For brevity, we denote matrix $\bLambda_{h, \bphi}(s) = \bphi(s, \pi^*_h(s))\bphi^\top(s, \pi^*_h(s)) \in \RR^{d \times d}$ and fix $h$, $m$ in the proof. The expectation $\EE_{s_h}[\bLambda_{h, \bphi}(s_h) | s_i], i < h$ is taken with respect to the randomness of the states sequence $s_{i+1}, \cdots, s_h$, where $s_{i' + 1} \sim \PP_{i'}\big(\cdot | s_{i'}, \pi^*_{i'}(s_{i'})\big), i \le i' < h$. If the action $a_i$ is given, the expectation $\EE_{s_h}[\bLambda_{h, \bphi}(s_h) | s_i, a_i], i < h$ is taken in which $s_{i+1} \sim \PP_i(\cdot | s_i, a_i)$ specially. It is worthless to show that $\EE_{s_h}[\bLambda_{h, \bphi}(s_h) | s_h^k] = \bLambda_{h, \bphi}(s_h^k)$. Without specification, we ignore the subscript $s_h$ in the expectation in the proof of this lemma.

% We reuse the notation $\bLambda$ and . Furthermore, we are assuming all of the expectation notation on $s_h$ is based on the distribution induced by optimal policy $\pi^*$ given the condition. For example, $\EE_{s_h}[\bLambda_{h, \bphi}(s_h) | s_i^k], i < h$ is calculated on the distribution of $s_h$ following the optimal policy from time-step $i$ to time-step $h - 1$ given the state at time-step $i$ is $s_i^k$. $\EE_{s_h}[\bLambda_{h, \bphi}(s_h) | s_i^k, a_i^k]$ is calculated on the distribution of $s_h$ following the optimal policy from time-step $i + 1$ to time-step $h - 1$ given at time-step $i$, the state-action pair is $(s_i^k, a_i^k)$. It is worthless to show that $\EE_{s_h}[\bLambda_{h, \bphi}(s_h) | s_h^k] = \bLambda_{h, \bphi}(s_h^k)$. We ignore the subscript $s_h$ in the expectation in the proof of this lemma.

To develop the convergence property of the summation $\bphi(s, a)\bphi^\top(s, a)$, we introduce the following matrix Azuma inequality.
% One of the key point to prove this lemma is like~\citet{papini2021leveraging} to craft the convergence of the matrix $\bphi(s, a)\bphi^\top(s, a)$ to the expectation defined in Assumption~\ref{asm:hls}. Thus the following matrix Azuma inequality is applied:
\begin{lemma}\label{lm:matazuma}[Matrices Azuma,~\citet{tropp2012user}]
Let $\{\cF_k\}_{k=1}^t$ be a filtration sequence, $\{\Xb_k\}_{k=1}^t$ be a finite adapted sequence of symmetric matrices where $\Xb_k \in \RR^{d \times d}$ is $\cF_{k+1}$-measurable, $\EE[\Xb_k | \cF_k] = \zero$ and $\Xb^2 \preceq \Cb^2$ a.s.. Then with probability at least $1 - \delta$, 
\begin{align*}
    \lambda_{\max}\left(\sum_{k=1}^t\Xb_k\right) \le \sqrt{8C^2t\log(d / \delta)},
\end{align*}
where $C = \|\Cb\|_2$.
\end{lemma}

Equipped with this lemma, we can start our proof.
\begin{proof}[Proof of Lemma~\ref{lm:lingrow}]
First it is easy to verify that for any $k \in [K]$
\begin{align}
    \bphi(s_h^k, a_h^k)\bphi^\top(s_h^k, a_h^k) &= \bLambda_{h, \bphi}(s_h^k) - \ind\left[a_h^k \neq \pi^*_h(s_h^k)\right]\left(\bLambda_{h, \bphi}(s_h^k) - \bphi(s_h^k, a_h^k)\bphi^\top(s_h^k, a_h^k)\right) \notag \\
    &\succeq \bLambda_{h, \bphi}(s_h^k) - C_{\bphi}d_{\bphi}\ind\left[a_h^k \neq \pi^*_h(s_h^k)\right]\Ib_{d_{\bphi}}, \label{eq:first}
\end{align}
where the inequality holds due to $\zero \preceq \bphi(s, a)\bphi^\top(s, a) \preceq (C_{\bphi}d_{\bphi})\Ib_{d_{\bphi}}$.
%due to $\|\bphi(s, a)\|_2^2 \le C_{\bphi}d_{\bphi}$ from Definition~\ref{asm:lin}.
By the definition of $\bLambda_{h, \bphi}(s_h^k)$, we have $\bLambda_{h, \bphi}(s_h^k) = \EE[\bLambda_{h, \bphi}(s_h) | s_h^k]$ 
%by definition
and  it suffices to control $\EE[\bLambda_{h, \bphi}(s_h) | s_i^k]$ for any $1 < i \le h$. Therefore it follows that%Due to the definition of MDP, 

\begin{align*}
    \EE[\bLambda_{h, \bphi}(s_h) | s_i^k] = \underbrace{\EE[\bLambda_{h, \bphi}(s_h) | s_{i - 1}^k, a_{i - 1}^k]}_{\Ab_i(s_{i - 1}^k, a_{i - 1}^k)} - \underbrace{\left(\EE[\bLambda_{h, \bphi}(s_h) | s_{i - 1}^k, a_{i - 1}^k] - \EE[\bLambda_{h, \bphi}(s_h) | s_{i}^k]\right)}_{\bepsilon_{i}^k}
\end{align*}
where we denote the first term as $\Ab_i(s_{i - 1}^k, a_{i - 1}^k)$ while the second term as $\bepsilon_i^k$ for simplicity. We first consider the term $\bepsilon_i^k$, it is easy to verify that $\bepsilon_i^k$ is $\cF_{i}^k$-measurable, $d\times d$ symmetric matrix with $\EE[\bepsilon_i^k | \cF_{i-1}^k] = \zero$ and
\begin{align}
     \|\bepsilon_i^k\|_2 \le \left\|\EE[\bLambda_{h, \bphi}(s_h) | s_{i - 1}^k, a_{i - 1}^k]\right\|_2 + \left\|\EE[\bLambda_{h, \bphi}(s_h) | s_{i}^k]\right\|_2 \le 2C_{\bphi}d_{\bphi}, \label{eq:current}
\end{align}
where the inequality holds due to the fact that  $\|\bLambda_{h, \bphi}(s)\|_2 \le C_{\bphi}d_{\bphi}$ .
%following proof similar to those in~\eqref{eq:first}.
Next, for the term $\Ab_i$, by introducing the indicator showing whether the action $a_{i-1}^k$ is the optimal action $\pi^*_{i-1}(s_{i-1}^k)$, we proceed as follows:
\begin{align}
    \Ab_i(s_{i-1}^k, a_{i-1}^k) &= \Ab_i(s_{i-1}^k, \pi^*_{i-1}(s_{i-1}^k)) \notag \\
    &\quad- \ind\left[a_{i-1}^k \neq \pi^*_{i-1}(s_{i-1}^k)\right]\left(\Ab_i(s_{i-1}^k, \pi^*_{i-1}(s_{i-1}^k)) -\Ab_i(s_{i-1}^k, a_{i-1}^k) \right) \notag \\
    &\succeq \Ab_i(s_{i-1}^k, \pi^*_{i-1}(s_{i-1}^k)) - C_{\bphi}d_{\bphi}\ind\left[a_{i-1}^k \neq \pi^*_{i-1}(s_{i-1}^k)\right]\Ib_{d_{\bphi}} \notag \\
    &= \EE[\bLambda_{h, \bphi}(s_h) | s_{i-1}^k]  - C_{\bphi}d_{\bphi}\ind\left[a_{i-1}^k \neq \pi^*_{i-1}(s_{i-1}^k)\right]\Ib_{d_{\bphi}}, \label{eq:forward}
\end{align}
where the inequality holds due to a similar proof of~\eqref{eq:first}. The last equality holds due to the definition that 
\begin{align*}
    \EE[\bLambda_{h, \bphi}(s_h)|s_{i-1}^k, \pi^*_{i-1}(s_{i-1}^k)] = \EE[\bLambda_{h, \bphi}(s_h) | s_{i - 1}].
\end{align*}
Combining~\eqref{eq:forward} and~\eqref{eq:current} together yields 
\begin{align*}
    \EE[\bLambda_{h, \bphi}(s_h) | s_i^k] \succeq \EE[\bLambda_{h, \bphi}(s_h) | s_{i-1}^k]  - C_{\bphi}d_{\bphi} \ind\left[a_{i-1}^k \neq \pi^*_{i-1}(s_{i-1}^k)\right]\Ib_{d_{\bphi}} - \bepsilon_i^k,
\end{align*}
and by telescoping over $i$ we have 
\begin{align}
    \EE[\bLambda_{h, \bphi}(s_h) | s_h^k] &\succeq \EE[\bLambda_{h, \bphi}(s_h) | s_1^k] - \sum_{i=2}^h\bepsilon_i^k - C_{\bphi}d_{\bphi}\sum_{i=1}^{h-1}\ind\left[a_i^k \neq \pi^*_i(s_i^k)\right]\Ib_{d_{\bphi}} \notag \\
    &= \EE[\bLambda_{h, \bphi}(s_h)] - \underbrace{\EE[\bLambda_{h, \bphi}(s_h)] - \EE[\bLambda_{h, \bphi}(s_h) | s_1^k]}_{\bepsilon_1^k} \notag \\
    &\quad- \sum_{i=2}^h\bepsilon_i^k - C_{\bphi}d_{\bphi}\sum_{i=1}^{h-1}\ind\left[a_i^k \neq \pi^*_i(s_i^k)\right]\Ib_{d_{\bphi}},\label{eq:revise-fin}
\end{align}
where $\bepsilon_1^k$ is $\cF_1^k$-measurable and $\EE[\bepsilon_1^k | \cF_H^{k-1}] = \zero, \|\bepsilon_1^k\|_2 \le 2C_{\bphi}d_{\bphi}$, which is similar to $\epsilon_i^k$ above.
Plugging~\eqref{eq:revise-fin} into~\eqref{eq:first} yields
\begin{align}
    \bphi(s_h^k, a_h^k)\bphi^\top(s_h^k, a_h^k) &\succeq \EE[\bLambda_{h, \bphi}(s_h)] - \sum_{i=1}^h\bepsilon_{i}^k - C_{\bphi}d_{\bphi}\sum_{i=1}^{h}\ind\left[a_i^k \neq \pi^*_i(s_i^k)\right]\Ib_{d_{\bphi}}, \notag \\
    &= \EE[\bLambda_{h, \bphi}(s_h)] - \sum_{i=1}^h\bepsilon_i^k - C_{\bphi}d_{\bphi}\sum_{i=1}^{h}\ind\left[Q^*_i(s_i^k, a_i^k) \neq V^*_i(s_i)\right]\Ib_{d_{\bphi}}, \notag \\
    &\succeq \EE[\bLambda_{h, \bphi}(s_h)] - \sum_{i=1}^h\bepsilon_i^k - \frac{C_{\bphi}d_{\bphi}}{\gap{\min}} \sum_{i=1}^h(V^*_i(s_i^j) - Q^*_i(s_i^j, a_i^j)) \Ib_{d_{\bphi}}\label{eq:single}
\end{align}
where the equality follows that $a_{h}^k \neq \pi^*_{i}(s_{h}^k)$ is equivalent with $Q^*_{i}(s_{i}^k, a_{h}^k) \neq V^*_{i}(s_{i})$ and the second inequality is from $V^*_{i}(s_{i}^k) - Q^*_{i}(s_{i}^k, a_{h}^k) \ge \gap{\min} \ind[Q^*_{i}(s_{i}^k, a_{h}^k) \neq V^*_{i}(s_{i})]$, which is according to the minimal sub-optimality gap condition assumption defined in~\eqref{def:gap}.
%, the second inequality comes from $\EE_{s_h}[\bLambda_h(s_h)] \succeq \lambda_h \Ib_{d_{\bphi}}$ in Assumption~\ref{asm:hls} and $V^*_{i}(s_{i}^k) - Q^*_{i}(s_{i}^k, a_{h}^k) \ge \gap{\min} \ind[Q^*_{i}(s_{i}^k, a_{h}^k) \neq V^*_{i}(s_{i})]$ according to the minimal sub-optimality gap condition assumption defined in~\eqref{def:gap}. \todoj{second inequality do not have term $\lambda_h$ or $ \gap{\min}$? Ignore some equation?}

By the construction of the ``covariance matrix'' $\Ub_{h, \bphi}^k$,~\eqref{eq:single} yields
\begin{align}
    \Ub_{h, \bphi}^k &=  \Ib_{d_{\bphi}} + \sum_{j=1}^{k-1} \bphi(s_h^j, a_h^j)\bphi^\top(s_h^j, a_h^j) \notag \\
    &\succeq \Ib_{d_{\bphi}} + (k - 1)\bLambda_{h, \bphi} - \sum_{j=1}^{k-1}\sum_{i=1}^h\bepsilon_i^j - \frac{C_{\bphi}d_{\bphi}}{\gap{\min}} \sum_{i=1}^h\sum_{j=1}^{k-1} (V^*_i(s_i^j) - Q^*_i(s_i^j, a_i^j)) \Ib_{d_{\bphi}}.\label{eq:regret}
\end{align}
Recall $\bepsilon_i^j$ is a $d_{\bphi}\times d_{\bphi}$ symmetric matrix, by Lemma~\ref{lm:matazuma} with $C = 2C_{\bphi}d_{\bphi}Ib_{d_{\bphi}}, t = (k-1)h$, with probability at least $1 - \delta$, we have
\begin{align}
   \lambda_{\max}\left(\sum_{j=1}^{k-1}\sum_{i=1}^h\bepsilon_i^j\right) \le \sqrt{32C^2_{\bphi}d_{\bphi}^2h(k-1)\log(d_{\bphi} / \delta)} \label{eq:8}.
\end{align}
Combining~\eqref{eq:8} with~\eqref{eq:regret} and substituting $\delta$ with $\delta/Hk(k+1)|\Phi|$, the claim in Lemma~\ref{lm:lingrow} holds for all $h \in [H], k \in [K], \bphi \in \Phi$ by taking a union bound.

\end{proof}

\subsection{Proof of Lemma~\ref{lm:gap}}

In order to prove Lemma~\ref{lm:gap}, we first need the following lemma.

% fast with respect to the episode number $k$.
% To prove this lemma, we need the following lemma showing that the UCB bonus term is decaying fast enough to induce such a constant regret bound

\begin{lemma}\label{lm:subdec}
Given the condition in Lemma~\ref{lm:he} and Lemma~\ref{lm:lingrow} holds and $\cE_{\bphi}^K$ holds for all $\bphi \in \Phi$. For each $\bphi \in \Phi$, there exists a constant threshold 
\begin{align*}
    \tau_{\bphi} = \text{poly}(d_{\bphi}, \sigma_{\bphi}^{-1}, H, \log(|\Phi| / \delta), \gap{\min}^{-1}, C_{\bphi}, C_{\bpsi}, C_{\Mb}, C_{\bpsi}')
\end{align*}
such that for any $(s, a) \in \cS \times \cA, h \in [H]$, there exists a representation candidate $\bphi \in \Phi$ where when $k \ge \tau_{\bphi}$,  $\bphi^\top(s, a)(\Ub_{h, \bphi}^k)^{-1}\bphi(s, a) \le 2C_{\bphi}d_{\bphi} / (\sigma_{\bphi}k)$. We denote $\tau = \max_{\bphi \in \Phi}{\tau_{\bphi}}$ to be the maximum possible threshold over all representations.
\end{lemma}
\noindent Lemma~\ref{lm:subdec} suggests that the UCB bonus term is decaying in the rate of $\cO\big(1/\sqrt{k}\big)$. Equipped with this lemma, we can start the proof.

\begin{proof}[Proof of Lemma~\ref{lm:gap}]
We will prove this lemma by induction. 
By the assumption in Lemma~\ref{lm:gap}, $\cE_{\bphi}^K$ holds for all $\bphi \in \Phi$. 
Considering $h = H$, for any state-action pair $(s, a) \in \cS \times \cA$, by Lemma~\ref{lm:subdec}, when $k \ge \tau$, there exists a representation $\bphi$ where Lemma~\ref{lm:gap1} yields
\begin{align*}
    Q_H^k(s, a) - Q^{\pi^k}_H(s, a) &\le 2C_{\bpsi} H\sqrt{\beta_{k, \bphi}\bphi^\top(s, a)(\Ub_{H, \bphi}^k)^{-1}\bphi(s, a)} + [\PP_H (V^k_{H+1} - V^\pi_{H+1})](s, a)\\
    &\le 2C_{\bpsi} H\sqrt{2C_{\bphi}d_{\bphi}\beta_{k, \bphi} / (\sigma_{\bphi}k)} + 0,
\end{align*}
where the second inequality is due to Lemma~\ref{lm:subdec} and the fact that $V_{H+1}^k, V_{H+1}^\pi$ are both equal to zero. Thus we have 
\begin{align*}
    \max_{(s, a) \in \cS \times \cA} \{Q_H^k(s, a) - Q^{\pi^k}_H(s, a)\} \le \max_{\bphi \in \Phi}\left\{2C_{\bpsi} H\sqrt{2C_{\bphi}d_{\bphi}\beta_{k, \bphi} / (\sigma_{\bphi}k)}\right\}.
\end{align*}
Suppose for step $h$, we have 
\begin{align}
    \max_{(s, a) \in \cS \times \cA} \{Q_h^k(s, a) - Q^{\pi^k}_h(s, a)\} \le (H - h + 1)\max_{\bphi \in \Phi}\left\{2C_{\bpsi} H\sqrt{2C_{\bphi}d_{\bphi}\beta_{k, \bphi} / (\sigma_{\bphi}k)}\right\}, \label{eq:indasm}
\end{align}
then considering time-step $h - 1$,  by Lemma~\ref{lm:gap1} and Lemma~\ref{lm:subdec}, for each $s, a$, there exists a $\bphi \in \Phi$ such that 
\begin{align*}
    &Q_{h-1}^k(s, a) - Q_{h-1}^{\pi^k}(s, a) \\
    &\quad\le 2C_{\bpsi} H\sqrt{\beta_{k, \bphi}\bphi^\top(s, a)(\Ub_{h-1, \bphi}^k)^{-1}\bphi(s, a)} + [\PP_{h-1} (V^k_h - V^{\pi^k}_h)](s, a)\\
    &\quad\le 2C_{\bphi}H\sqrt{2C_{\bphi}d_{\bphi}\beta_{k, \bphi} / (\sigma_{\bphi}k)} + [\PP_{h-1} (Q^k_h(\cdot, \pi^k_h(\cdot)) - Q^{\pi^k}_h(\cdot, \pi^k_h(\cdot)))](s, a)\\
    &\quad \le 2C_{\bphi}H\sqrt{2C_{\bphi}d_{\bphi}\beta_{k, \bphi} / (\sigma_{\bphi}k)} + (H - h + 1)\max_{\bphi \in \Phi}\left\{2C_{\bpsi} H\sqrt{2C_{\bphi}d_{\bphi}\beta_{k, \bphi} / (\sigma_{\bphi}k)}\right\}\\
    &\quad \le (H - h+2)\max_{\bphi \in \Phi}\left\{2C_{\bpsi} H\sqrt{2C_{\bphi}d_{\bphi}\beta_{k, \bphi} / (\sigma_{\bphi}k)}\right\}.
\end{align*}
where the second inequality follows from the definition that $V_h^k(s) = Q_h^k(s, \pi_h^k(s))$ and $V_h^{\pi^k}(s) = Q_h^{\pi^k}(s, \pi_h^k(s))$, the third inequality is due to the induction assumption~\eqref{eq:indasm} and  this result conclude our induction.

Then, following Lemma~\ref{lm:optim}, we have $Q_{h, \bphi}^k(s, a) \ge Q^*_h(s, a)$. Thus, $Q_h^k(s, a) = \min_{\bphi \in \Phi}\{Q^k_{h, \bphi}\} \ge Q^*_h(s, a)$. Then the sub-optimality gap could be bounded by 
\begin{align*}
    \gap{h}(s, \pi_h^k(s)) &= Q^*_h(s, \pi^*_h(s)) - Q^*_h(s, \pi_h^k(s))\\
    &\le Q^k_h(s, \pi^*_h(s)) - Q^{\pi^k}_h(s, \pi_h^k(s))\\
    &\le Q^k_h(s, \pi^k_h(s)) - Q^{\pi^k}_h(s, \pi_h^k(s))\\
    &\le (H - h + 1)\max_{\bphi \in \Phi}\left\{2C_{\bpsi} H\sqrt{2C_{\bphi}d_{\bphi}\beta_{k, \bphi} / (\sigma_{\bphi}k)}\right\}\\
    &\le 2C_{\bpsi}H^2\max_{\bphi \in \Phi}\left\{\sqrt{2C_{\bphi}d_{\bphi}\beta_{k, \bphi} / (\sigma_{\bphi}k)}\right\},
\end{align*}
where the inequality on the second line holds due to Lemma~\ref{lm:optim}, and the inequality on the third line holds due to the greedy policy $\pi_h^k(s) = \argmax_a Q_h^k(s, a)$. Finally, the inequality on the forth line is due to the result of induction~\eqref{eq:indasm} and we finish the proof.
\end{proof}

\section{Proof of Lemmas in Appendix~\ref{app:proof}}

\subsection{Proof of Lemma~\ref{lm:optim}}
\begin{lemma}[Lemma 5 on $B_n^{(2)}$, pp. 23,~\citet{yang2020reinforcement}]\label{lm:5}
Suppose $\cE_{\bphi}^K$ holds, then for any $(s, a) \in \cS \times \cA$, we have 
\begin{align*}
    \|\bphi(s, a)^\top(\Mb_{h, \bphi}^k - \Mb_{h, \bphi}^*)\|_2 \le \sqrt{\beta_{k, \bphi}\bphi^\top(s, a)(\Ub_{h, \bphi}^k)^{-1}\bphi(s, a)}.
\end{align*}
\end{lemma}
\begin{proof}[Proof of Lemma~\ref{lm:optim}]
We prove this lemma by induction. First, it is obvious that $Q_{H+1}^k(s, a) = Q_{H+1}^*(s, a) = 0$ for all $(s, a) \in \cS \times \cA$. Then assuming for $1 < h \le H$, we have $Q_{h+1}^k(s, a) \ge Q_{h+1}^*(s, a)$ holds for all $(s, a)$, considering time-step $h$ and representation $\bphi$, we have 
\begin{align}
    Q_{h, \bphi}^k(s, a) &= r(s, a) + \bphi^\top(s, a)\Mb_{h, \bphi}^k\bPsi^\top \vb_{h+1}^k + C_{\bpsi} H\sqrt{\beta_{k, \bphi}\bphi^\top(s, a)(\Ub_{h, \bphi}^k)^{-1}\bphi(s, a)}\notag\\
    &= r(s, a) + \bphi^\top(s, a)(\Mb_{h, \bphi}^k - \Mb_{h, \bphi}^*)\bPsi^\top \vb_{h+1}^k\notag\\
    &\quad + C_{\bpsi} H\sqrt{\beta_{k, \bphi}\bphi^\top(s, a)(\Ub_{h, \bphi}^k)^{-1}\bphi(s, a)} + [\PP_h V_{h+1}^k](s, a)\notag\\
    &\ge r(s, a) - \|\bphi^\top(s, a)(\Mb_{h, \bphi}^k - \Mb_{h, \bphi}^*)\|_2\|\bPsi^\top \vb_{h+1}^k\|_2\notag\\
    &\quad + C_{\bpsi} H\sqrt{\beta_{k, \bphi}\bphi^\top(s, a)(\Ub_{h, \bphi}^k)^{-1}\bphi(s, a)} + [\PP_h V_{h+1}^k](s, a) \notag \\
    &\ge r(s, a) + [\PP_h V_{h+1}^k](s, a),\label{eq:temp2}
\end{align}
where the first inequality comes from the fact that $\la \xb, \yb \ra \ge -\|\xb\|_2 \|\yb\|_2$, the second inequality holds due to Lemma~\ref{lm:5} and $\|\bPsi^\top \vb_{h+1}^k\|_{\infty} \le C_{\bpsi}H$ since $\|\vb_{h+1}^k\|_{\infty} \le H$. Since $Q_{h+1}^k(s, a) \ge Q_{h+1}^*(s, a)$, then
\begin{align*}
    V_{h+1}^k(s) &= \min\{H, Q_{h+1}^k(s, \pi_h^k(s))\} \\
    &\ge \min\{H, Q_{h+1}^k(s, \pi_{h+1}^*(s))\} \\
    &\ge \min\{H, Q_{h+1}^*(s, \pi_{h+1}^*(s))\} \\
    &= V_{h+1}^*(s),
    \end{align*}%\todoj{$V^k=min(Q^k,H)$, we cannot say $V^k\ge Q^k$},
where the last inequality is due to the fact that $V_{h+1}^*(s) \le H$. Therefore,~\eqref{eq:temp2} yields $Q_{h, \bphi}^k(s, a) \ge r(s, a) + [\PP_h V^*_{h+1}](s, a) = Q^*_h(s, a)$ for all $\bphi \in \Phi$. Thus
\begin{align*}
    Q_h^k(s, a) = \min_{\bphi \in \Phi}\{Q_{h, \bphi}^k(s, a)\} \ge Q^*_h(s, a).
\end{align*}
Then we finish our proof by induction.
\end{proof}

\subsection{Proof of Lemma~\ref{lm:gap1}}
\begin{proof}[Proof of Lemma~\ref{lm:gap1}]
First, the update rule of $Q_{h, \bphi}^k$ and Bellman equation yield
\begin{align}
    Q_{h, \bphi}^k(s, a) - Q^{\pi}_h(s, a) &= \underbrace{\bphi^\top(s, a)\Mb_{h, \bphi}^k\bPsi^\top \vb_{h+1}^k}_{I_1} - [\PP_h V^\pi_{h+1}](s, a) \notag  \\
    &\quad + C_{\bpsi} H\sqrt{\beta_{k, \bphi}\bphi^\top(s, a)(\Ub_{h, \bphi}^k)^{-1}\bphi(s, a)}.\label{eq:t3}
\end{align}
Since $[\PP_h V_{h+1}^k] = \bphi^\top(s, a)\Mb_{h, \bphi}^*\bPsi^\top \vb_{h+1}^k$, $I_1$ can be decomposed as
\begin{align}
    \bphi^\top(s, a)\Mb_{h, \bphi}^k\bPsi^\top \vb_{h+1}^k &= \bphi^\top(s, a)(\Mb_{h, \bphi}^k - \Mb^*_{h, \bphi})\bPsi^\top \vb_{h+1}^k + [\PP_h V_{h+1}^k](s, a) \notag \\
    &\le \|\bPsi^\top\vb_{h+1}^k\|_2 \|\bphi^\top(s, a)(\Mb_{h, \bphi}^k - \Mb_{h, \bphi}^*)\|_2 + [\PP_h V_{h+1}^k](s, a) \notag \\
    &\le C_{\bpsi}H\sqrt{\beta_{k, \bphi}\bphi^\top(s, a)(\Ub_{h, \bphi}^k)^{-1}\bphi(s, a)} + [\PP_h V_{h+1}^k](s, a), \label{eq:t4}
\end{align}
where the inequality on the second line holds due to $\la \xb, \yb\ra \le \|\xb\|_2\|\yb\|_2$ and the inequality on the third line comes from Lemma~\ref{lm:5} with $\|\vb_{h+1}^k\|_{\infty} \le H$ and Definition~\ref{asm:lin}. Plugging~\eqref{eq:t4} into~\eqref{eq:t3} yields
\begin{align*}
    Q_{h, \bphi}^k(s, a) - Q^{\pi}_h(s, a) \le 2C_{\bpsi}H\sqrt{\beta_{k, \bphi}\bphi^\top(s, a)(\Ub_{h, \bphi}^k)^{-1}\bphi(s, a)} + [\PP_h (V_{h+1}^k - V^{\pi}_{h+1})](s, a).
\end{align*}
Since $Q_h^k(s, a) = \min_{\bphi \in \Phi}(s, a)$, we can get the claimed result in Lemma~\ref{lm:gap1}.
\end{proof}

\subsection{Proof of Lemmas~\ref{lm:he62}}
\begin{lemma}[Lemma 6.6,~\citet{he2020logarithmic}]\label{lm:66}
For any subset $C = \{c_1, \cdots, c_k\} \subseteq [K]$ and any $h \in [H]$,
\begin{align*}
    \sum_{i=1}^k \bphi_h^\top(s_h^{c_i}, a_h^{c_i})(\Ub_{h, \bphi}^{c_i})^{-1}\bphi_h(s_h^{c_i}, a_h^{c_i}) \le 2d_{\bphi}\log(1 + C_{\bphi}kd_{\bphi})
\end{align*}
\end{lemma}
\begin{remark}
Proof of Lemma~\ref{lm:66} remains the same as~\citet{he2020logarithmic} by changing the norm of $\bphi$ from $\|\bphi\|_2^2 \le 1$ to $\|\bphi\|_2^2 \le C_{\bphi}d_{\bphi}$ as Definition~\ref{asm:lin}.
\end{remark}
\begin{lemma}[Azuma-Hoeffding's inequality,~\citealt{azuma1967weighted}]\label{lm:azuma}
Let $\{x_i\}_{i=1}^n$ be a martingale difference sequence with respect to a filtration $\{\cF_i\}_{i=1}^n$ (i.e. $\EE[x_i | \cF_i] = 0$ a.s. and $x_i$ is $\cF_{i+1}$ measurable) such that $|x_i| \le M$ a.s.. Then for any $0 < \delta < 1$, with probability at least $1 - \delta$, $\sum_{i=1}^n x_i \le M\sqrt{2n\log(1 / \delta)}$.
\end{lemma}
\begin{proof}[Proof of Lemma~\ref{lm:he62}]
We fix $h$ and consider the first $k$ episodes in this proof. Let $k_0 = 0$, for any $j \in [k]$, we denote $k_j$ as the minimum index of the episode where the sub-optimality at time-step $h$ is no less than $\Delta$:
\begin{align*}
    k_j = \min\left\{\bar k : \bar k > k_{j-1}, V^*_h(s_h^{\bar k}) - Q_h^{\pi^{\bar k}}(s_h^{\bar k},a_h^{\bar k}) \ge \Delta\right\}.
\end{align*}
For simplicity, we denote $k'$ to be the number of episodes such that the sub-optimality of this episode at step $h$ is no less than $\Delta$, i.e.
\begin{align*}
    k' = \sum_{j=1}^{k} \ind[V_h^*(s_h^j) - Q_h^{\pi^j}(s_h^j, a_h^j) \ge \Delta].
\end{align*}
Then by the definition of $k'$, it is obvious that 
\begin{align}
    \sum_{j=1}^{k'} Q_h^{k_j}(s_h^{k_j}, a_h^{k_j}) - Q_h^{\pi^{k_j}}(s_h^{k_j}, a_h^{k_j}) &\ge \sum_{j=1}^{k'} Q_h^{k_j}(s_h^{k_j}, \pi^*_h(s_h^{k_j})) - Q_h^{\pi^{k_j}}(s_h^{k_j}, a_h^{k_j}) \notag \\
    &\ge \sum_{j=1}^{k'} Q_h^*(s_h^{k_j}, \pi^*_h(s_h^{k_j})) - Q_h^{\pi^{k_j}}(s_h^{k_j}, a_h^{k_j}) \notag \\
    &= \sum_{j=1}^{k'} V_h^*(s_h^{k_j}) - Q_h^{\pi^{k_j}}(s_h^{k_j}, a_h^{k_j}) \ge \Delta k' \label{eq:q1},
\end{align}
where the first inequality holds due to $a_h^k = \argmax_a Q_h^k(s_h^k, a)$ and the second inequality follows Lemma~\ref{lm:optim}. On the other hand, following Lemma~\ref{lm:gap1}, when $\cE_{\phi}^k$ holds, for all $i \in [H], j \le k$ 
\begin{align}
    Q_i^j(s_i^j, a_i^j) &\le 2C_{\bpsi}H\sqrt{\beta_{j, \bphi}\bphi^\top(s_i^j, a_i^j)(\Ub_{i, \bphi}^j)^{-1}\bphi(s_i^j, a_i^j)} + [\PP_i (V_{i+1}^j - V^{\pi^j}_{i+1})](s_i^j, a_i^j) \notag \\
    &= 2C_{\bpsi}H\sqrt{\beta_{j, \bphi}\bphi^\top(s_i^j, a_i^j)(\Ub_{i, \bphi}^j)^{-1}\bphi(s_i^j, a_i^j)} + V_{i+1}^j(s_{i+1}^j) - V^{\pi^j}_{i+1}(s_{i+1}^j) + \epsilon_i^j, \label{eq:t5}
\end{align}
where $\epsilon_i^j = [\PP_i (V_{i+1}^j - V^{\pi^j}_{i+1})](s_i^j, a_i^j) - \big(V_{i+1}^j(s_{i+1}^j) - V^{\pi^j}_{i+1}(s_{i+1}^j)\big)$. It is easy to verify that $|\epsilon_i^j| \le H$, $\epsilon_i^j$ is $\cF_{i+1}^j$ measurable with $\EE[\epsilon_i^j | \cF_{i}^j] = 0$. Taking the telescoping summation on~\eqref{eq:t5} over $h \le i \le H, j \in \{k_1, \cdots, k_{k'}\}$ using the fact that $V_i^j(s_i^j) = Q_i^j(s_i^j, a_i^j)$ and $V_i^{\pi^j}(s_i^j) = Q_i^{\pi^j}(s_i^j, a_i^j)$ we have 
\begin{align}
    \sum_{j=1}^{k'} Q_h^{k_j}(s_h^{k_j}, a_h^{k_j}) - Q_h^{\pi^{k_j}}(s_h^{k_j}, a_h^{k_j}) &\le I_1 + I_2, \label{eq:q2}
\end{align}
where
\begin{align}
    I_1 &= \sum_{j=1}^{k'}\sum_{i=h}^H 2C_{\bpsi}H\sqrt{\beta_{k_j, \bphi}\bphi^\top(s_h^{k_j}, a_h^{k_j})(\Ub_{i, \bphi}^{k_j})^{-1}\bphi(s_h^{k_j}, a_h^{k_j})} \notag \\
    I_2 &= \sum_{j=1}^{k'}\sum_{i=h}^H\epsilon_i^{k_j}. \notag 
\end{align}
To bound $I_1$, by Cauchy-Schwarz inequality, 
\begin{align*}
    I_1 &= \sum_{j=1}^{k'}\sum_{i=h}^H 2C_{\bpsi}H\sqrt{\beta_{k_j, \bphi}\bphi^\top(s_h^{k_j}, a_h^{k_j})(\Ub_{i, \bphi}^{k_j})^{-1}\bphi(s_h^{k_j}, a_h^{k_j})}\\
    &\le 2C_{\bpsi}H\sqrt{\beta_{k, \bphi}k'}\sum_{i=h}^H\sqrt{\sum_{j=1}^{k'}\bphi^\top(s_h^{k_j}, a_h^{k_j})(\Ub_{i, \bphi}^{k_j})^{-1}\bphi(s_h^{k_j}, a_h^{k_j})}\\
    &\le 2C_{\bpsi} H^2\sqrt{\beta_{k, \bphi}k'}\sqrt{2d_{\bphi}\log(1 + C_{\bphi}k'd_{\bphi})}\\
    &\le 2C_{\bpsi}H^2\sqrt{2\beta_{k, \bphi}d_{\bphi}k'\log(1 + C_{\bphi}kd_{\bphi})},
\end{align*}
where the second inequity in Line 3 is from Lemma~\ref{lm:66}. To bound $I_2$, by Lemma~\ref{lm:azuma}, with probability at least $1 - \delta / k$, we have 
\begin{align*}
    \sum_{j=1}^{k'}\sum_{i=h}^H \epsilon_i^{k_j} \le \sqrt{2k'H^3\log(k / \delta)},
\end{align*}
then taking union bound over all $k$ we can conclude that with probability at least $1 - \delta$,
\begin{align*}
    I_2 = \sum_{j=1}^{k'}\sum_{i=h}^H \epsilon_i^{k_j} \le \sqrt{2k'H^3\log(k / \delta)}
\end{align*}
Combining~\eqref{eq:q1} with~\eqref{eq:q2}, we can obtain
\begin{align}
    \Delta k' \le 2C_{\bpsi}H^2\sqrt{2\beta_{k, \bphi}d_{\bphi}k'\log(1 + C_{\bphi}kd_{\bphi})} + \sqrt{2k'H^3\log(k / \delta)}.\label{eq:temp1}
\end{align}
By $(a + b)^2 \le 2a^2 + 2b^2$,~\eqref{eq:temp1} immediately implies 
\begin{align}
    k' \le \frac{16C_{\bpsi}^2H^4d_{\bphi}\beta_{k, \bphi}\log(1 + C_{\bphi}kd) + 4H^3\log(k / \delta)}{\Delta^2}\label{eq:tt1}
\end{align}
Since event $\cE_{\bphi}^K$ directly implies $\cE_{\bphi}^k$ for all $k \le K$, we can get the claimed result~\eqref{eq:tt1} holds for all $k \le K$ with probability $1 - \delta$. Replace $\delta$ with $\delta / k(k+1)$ for different $k$, taking union bound for all possible $k$, we have with probability at least $1 - \delta$, for all possible $k$, 
\begin{align*}
    \sum_{j=1}^k \ind[V_h^*(s_h^j) - Q_h^{\pi^j}(s_h^j, a_h^j)] &\le \frac{16C_{\bpsi}^2H^4d_{\bphi}\beta_{k, \bphi}\log(1 + C_{\bphi}kd_{\bphi}) + 4H^3\log(k^2(k+1) / \delta)}{\Delta^2}\\
    &\le \frac{16C_{\bpsi}^2H^4d_{\bphi}\beta_{k, \bphi}\log(1 + C_{\bphi}kd_{\bphi}) + 12H^3\log(2k / \delta)}{\Delta^2}.
\end{align*}
\end{proof}

\subsection{Proof of Lemma~\ref{lm:subdec}}
\begin{proof}[Proof of Lemma~\ref{lm:subdec}]
For any state-action pair $(s, a)$ at step $h$, according to Assumption~\ref{asm:hls}, we consider the set $ \cZ_{h, \bphi}$ where $(s, a) \in \cZ_{h, \bphi}$ and the corresponding representation $\bphi$. By Lemma~\ref{lm:lingrow}, we denote $\Bb$ as 
\begin{align*}
        \Bb &:=  (k-1)\bLambda_{h, \bphi} - \iota\Ib_{d_{\bphi}} \preceq \Ub^k_{h, \bphi},\\
        \iota &= \frac{C_{\bphi}d_{\bphi}}{\gap{\min}}\sum_{i=1}^h\sum_{j=1}^{k-1} \gap{i}(s_i^j, a_i^j) + C_{\bphi}d_{\bphi}\sqrt{32H(k-1)\log(d_{\bphi}|\Phi|Hk(k+1) / \delta)} - 1.
    \end{align*}
Decomposing $\bLambda_{h, \bphi} = \Qb^\top \Db\Qb$ where $\Qb \in \RR^{d_{\bphi} \times d_{\bphi}}$ is the orthogonal matrix and $\Db$ is the diagonal matrix, we have $\Bb = \Qb^\top((k-1)\Db - \iota \Ib_{d_{\bphi}})\Qb$.

We first prove the non-singular property of $\Bb$. Considering the zero diagonal element $\Db_{[ii]}$, we have 
\begin{align*}
    ((k-1)\Db - \iota \Ib_{d_{\bphi}})_{[ii]} \le -\iota \le  - C_{\bphi}d_{\bphi}\sqrt{32H(k-1)\log(d_{\bphi}|\Phi|Hk(k+1) / \delta)} + 1,
\end{align*}
where the second inequality is due to $\gap{h}(s, a) \ge 0$. As a result, it is obvious to verify that there exists a constant $K_1$ such that once $k \ge K_1$, $((k-1)\Db - \iota \Ib_{d_{\bphi}})_{[ii]} < 0$ for all zero diagonal element $\Db_{[ii]}$ in $\Db$. Next we consider the non-zero diagonal value $\Db_{[jj]}$. By Assumption~\ref{asm:hls}, $\Db_{[jj]} \ge \sigma_{\bphi}$. Therefore, the corresponding diagonal value $((k-1)\Db - \iota \Ib_{d_{\bphi}})_{[jj]}$ could be bounded by
\begin{align*}
    ((k-1)\Db - \iota \Ib_{d_{\bphi}})_{[jj]} \ge \sigma_{\bphi}(k-1) - \iota.
\end{align*}
Removing the minimum operator in~\eqref{eq:hegap} in Lemma~\ref{lm:he}, we have 
\begin{align*}
    ((k-1)\Db - \iota \Ib_{d_{\bphi}})_{[jj]} &\ge 1 + \sigma_{\bphi}(k-1) - C_{\bphi}d_{\bphi}\sqrt{32H(k-1)\log(d_{\bphi}|\Phi|Hk(k+1) / \delta)} \\
    &\quad - \frac{64C_{\bpsi}^2H^4d_{\bphi}^2\beta_{k, \bphi}\log(1 + C_{\bphi}{k}d_{\bphi}) + 48H^3\log\big(2k (1 + \log(H / \gap{\min})/ \delta\big)}{\gap{\min}} %\\
    % &\ge 1 + \sigma_{\bphi}(k-1) - C_{\bphi}d_{\bphi}\sqrt{32H(k-1)\log(d_{\bphi}|\Phi|Hk(k+1) / \delta)} \\
    % &\quad - \frac{64C_{\bpsi}^2H^4d_{\bphi}^2\beta_{k, \bphi}\log(1 + C_{\bphi}{k}d_{\bphi}) + 48H^3\log\big(2k (1 + \log(H / \gap{\min})/ \delta\big)}{\gap{\min}},
\end{align*}
%where the second inequality removes the $\min$ operator.
It's easy to verify that the increasing term $\sigma_{\bphi}k$ is $\cO(k)$ while the decreasing term is in the order of $\cO(\sqrt{k})$ and $\cO(\log(k))$ where $\beta_{k, \bphi} = \cO(\log(k))$ as shown in Lemma~\ref{lm:good}, thus there exists a constant threshold \begin{align*}
\tau_{\phi} = \text{poly}(d_{\bphi}, \sigma_{\bphi}^{-1}, H, \log(|\Phi| / \delta), \gap{\min}^{-1}, C_{\bphi}, C_{\bpsi}, C_{\Mb}, C_{\bpsi}')
\end{align*}
such that for any $k \ge \tau_{\bphi}$, $((k-1)\Db - \iota \Ib_{d_{\bphi}})_{[jj]} \ge \sigma_{\bphi} / 2$. Since we have shown that all of the diagonal value for $(k-1)\Db - \iota$ is either strictly smaller than zero or strictly greater than zero, $\Bb$ is invertible.

By the definition of $\cZ_{h, \bphi}$ in Assumption~\ref{asm:hls}, there exists a vector $\xb \in \RR^{d_{\bphi}}$ such that $\bLambda \xb = \bphi(s, a) / \|\bphi(s, a)\|_2$. Since $\Ub_{h, \bphi}^k \succeq \Bb$ and $\Ub_{h, \bphi}^k, \Bb$ are both invertible, it follows 
\begin{align}
    \bphi^\top(s, a)(\Ub_{h, \bphi}^k)^{-1}\bphi(s, a) \le \|\bphi(s, a)\|_2^2 \underbrace{\frac{\bphi^\top(s, a)}{\|\bphi(s, a)\|_2}\Bb^{-1}\frac{\bphi(s, a)}{\|\bphi(s, a)\|_2}}_{I_1}, \label{eq:s3}
\end{align}
where $I_1$ could be rewrote by
\begin{align}
    I_1 &= \xb^\top \bLambda \Bb^{-1}\bLambda\xb \notag \\
    &= \xb^\top \Qb^\top \Db \Qb \Qb^\top ((k - 1)\Db - \iota \Ib_{d_{\bphi}})^{-1}\Qb\Qb^\top \Db \Qb \xb \notag \\
    &= \xb^\top \Qb^\top \Db ((k - 1)\Db - \iota \Ib_{d_{\bphi}})^{-1} \Db \Qb \xb. \label{eq:s2}
\end{align}
Since $\|\bLambda\xb\|_2 = 1$, it is easy to verify that $\xb^\top \Qb^\top \Db \Db \Qb \xb = \xb^\top \bLambda \bLambda \xb = \|\bLambda \xb\|_2^2 = 1$. We hereby denote $\yb$ as $\Db \Qb \xb$ and we have $\|\yb\|_2^2 = 1$. Furthermore, it is obvious that $\yb_{[i]} = 0$ as long as $\Db_{[ii]} = 0$. Therefore, $\sum_{i=1, \Db_{[ii]} \neq 0}^{d_{\bphi}} \yb^2_{[i]} = 1$. Then plugging the notation of $\yb$ into~\eqref{eq:s2} yields
\begin{align*}
    I_1 = \yb^\top ((k - 1)\Db - \iota \Ib_{d_{\bphi}})^{-1} \yb = \sum_{i=1, \Db_{[ii]} \neq 0}^{d_{\bphi}} \frac{\yb^2_{[i]}}{((k - 1)\Db - \iota \Ib_{d_{\bphi}})_{[ii]}},
\end{align*}
since we have shown that the $((k - 1)\Db - \iota \Ib_{d_{\bphi}})_{[ii]} \ge \sigma_{\bphi}k / 2$ when $\Db_{[ii]} \neq 0$ and $k \ge \tau_{\bphi}$. Thus we can easily conclude that $I_1 \le 2/(\sigma_{\bphi}k)$, plugging this into~\eqref{eq:s3} we can get the claimed result.% where $\tau_{\bphi} = \text{poly}(d_{\bphi}, \sigma_{\bphi}^{-1}, H, \log(|\Phi| / \delta), \gap{\min}^{-1}, C_{\bphi}, C_{\bpsi}, C_{\Mb}, C_{\bpsi}')$. 
\end{proof}

\section{Proof of Theorem~\ref{thm:offline}}\label{sec:app-offline-1}
In this section, we provide the proof of Theorem~\ref{thm:offline}, which bounds the sample complexity of the offline version algorithm \algname-LCB. The offline training process favors a similar ``good event'' with its online counterpart (Lemma~\ref{lm:good}) which is formalized as the lemma below:
\begin{lemma}[Lemma 15, \citet{yang2020reinforcement}, offline ver.]\label{lm:good-offline}
Define the following event as $\cE_{\bphi}^k$: 
\begin{align*}
    \left\{\tr\left[(\Mb_{h, \bphi} - \Mb^*_{h, \bphi})^\top \Ub_{h, \bphi}(\Mb_{h, \bphi} - \Mb^*_{h, \bphi})\right] \le \beta_{\bphi}, \forall h \in [H]\right\} =: \cE_{\bphi}.
\end{align*}
With $\beta_{\bphi} = Cd_{\bphi}\log(KH/\delta)$ for some absolute constant $C > 0$, we have $\Pr(\cE_{\bphi}) \ge 1 - \delta$ for all $\bphi \in \Phi$.
\end{lemma}
\begin{proof}
The proof is similar with the original proof in~\citet{yang2020reinforcement} by changing $k$ to $K$. The remaining part is unchanged given the offline training data.
\end{proof}

Then the next lemma is esentially the first part of Theorem~\ref{thm:offline}, which provides an upper bound of the sub-optimality planned by Algorithm~\ref{alg:main-offline}.

\begin{lemma}\label{lm:err-offline}
Let $\beta$ set as Lemma~\ref{lm:good-offline}. If the event $\cE_{\bphi}$ in Lemma~\ref{lm:good-offline} holds for all $\bphi \in \Phi$, for all state $s \in \cS$ and $h \in [H]$,
\begin{align*}
    V_h^*(s) - V_h^\pi(s) \le 2C_{\bpsi}H\sum_{h' = h}^H\EE_{\pi^*}\left[\min_{\bphi \in \Phi}\left\{\sqrt{\beta_{\bphi}}\|\bphi(s, a)\|_{\Ub_{h', \bphi}^{-1}}\right\}\middle | s_h = s\right],
\end{align*}
where the expectation is taken with respect to the trajectory induced by the optimal policy $\pi^*$ given the fixed covariance matrix $\Ub_{h, \bphi}$.
\end{lemma}

Comparing with~\citet{jin2021pessimism}, our results adapts the minimal uncertainty $\|\bphi\|_{\Ub^{-1}}$ over all representation $\bphi \in \Phi$. Therefore, even if each single representation $\bphi$ cannot satisfy Assumption~\ref{asm:offline}, we can still get sample complexity bound which~\citet{jin2021pessimism} failed to provide. 

Then the next lemma suggests that the uncertainty $\|\bphi\|_{\Ub^{-1}}$ is bounded by $\tilde \cO(1 / \sqrt{K})$ where the $K$ is the size of the offline dataset.

\begin{lemma}\label{lm:sublin-offline}
With probability at least $1 - \delta$, for $(s, a, h) \in \cS \times \cA \times [H]$, there exists a $\bphi \in \Phi$ such that when 
\begin{align}
    K > \frac{32C_{\bphi}^2d_{\bphi}^2\log(Hd_{\bphi}|\Phi|/\delta)}{\tilde \sigma^{2}_{h, \bphi}}\left(1 + \frac{C_{\bpsi}^2H^4\beta_{\bphi}C_{\bphi}\tilde \sigma_{h, \bphi}}{4\gap{\min}^2C_{\bphi}^2d_{\bphi}\log(Hd_{\bphi}|\Phi|/\delta)}\right),\label{eq:thres-offline}
\end{align}
we have $\|\bphi(s, a)\|_{\Ub_{h, \bphi}^{-1}} < \gap{\min}/ (2H^2C_{\bpsi}\sqrt{\beta_{\bphi}})$. Here $\tilde \sigma_{h, \bphi}$ is the minimum non-zero eigen value of expected offline matrix $\EE_{d_h^{\hat \pi}}[\bphi\bphi^\top]$ and $K$ is the number of trajectories in offline data. 
\end{lemma}

Equipped with these lemmas, we can start our proof.
\begin{proof}[Proof of Theorem~\ref{thm:offline}]
The proof for the first part of the theorem have been shown in Lemma~\ref{lm:err-offline}, where we assume the event $\cE_{\bphi}$ in Lemma~\ref{lm:good-offline} holds. Then suppose the event in Lemma~\ref{lm:sublin-offline} holds, let $K$ be greater than the threshold~\eqref{eq:thres-offline} provided in Lemma~\ref{lm:sublin-offline}, i.e. 
\begin{align}
    K > \max_{\bphi \in \Phi, h \in [H]} \left\{ \frac{32C_{\bphi}^2d_{\bphi}^2\log(Hd_{\bphi}|\Phi|/\delta)}{\tilde \sigma^{2}_{h, \bphi}}\left(1 + \frac{C_{\bpsi}^2H^4\beta_{\bphi}C_{\bphi}\tilde \sigma_{h, \bphi}}{4\gap{\min}^2C_{\bphi}^2d_{\bphi}\log(Hd_{\bphi}|\Phi|/\delta)}\right)\right\},\label{eq:k-cond}.
\end{align}
By Lemma~\ref{lm:sublin-offline}, for all $(s, a, h) \in \cS \times \cA \times [H]$, there exists a $\bphi \in \Phi$ such that $\|\bphi(s, a)\|_{\Ub_{h, \bphi}^{-1}} < \Delta/ (2H^2C_{\bpsi}\sqrt{\beta_{\bphi}})$. Then by Lemma~\ref{lm:err-offline}, for any state $s \in \cS$ at step $h \in [H]$, the sub-optimality is bounded by
\begin{align}
    V_h^*(s) - V_h^\pi(s) &\le 2C_{\bpsi}H\sum_{h' = h}^H \EE_{\pi^*}\left[\min_{\bphi \in \Phi}  \left\{\sqrt{\beta_{\bphi}}\|\bphi(s, a)\|_{\Ub_{h', \bphi}^{-1}}\right\}\middle | s_h = s\right] \notag \\
    &< \sum_{h' = h}^H\EE_{\pi^*}\left[ \frac{\gap{\min}}{H} \middle | s_h = s\right] \notag \\
    &< \gap{\min}.\label{eq:offline1}
\end{align}
On the other hand, since $Q^*_h(s, \pi(s)) \ge V_h^\pi(s)$, by the definition of sub-optimality gap in Definition~\ref{def:gap},
\begin{align}
    V_h^*(s) - V_h^\pi(s) \ge V_h^*(s) - Q_h^*(s, \pi(s)) \ge \ind[\pi_h(s) \neq \pi^*_h(s)]\gap{\min}, \label{eq:offline2}
\end{align}
where the last inequality follows the uniqueness of the optimal policy and all other action will lead to sub-optimality. Combining~\eqref{eq:offline1} and~\eqref{eq:offline2} together suggests that when $K$ satisfies the condition in~\eqref{eq:k-cond},
\begin{align*}
    \ind[\pi_h(s) \neq \pi^*_h(s)]\gap{\min} < \gap{\min},
\end{align*}
which yields that $\pi_h(s) = \pi_h^*(s)$. Applying this to all $(s, h) \in \cS \times [H]$ and replacing $\delta$ with $\delta / (2|\Phi|)$, we can get the claimed result in Theorem~\ref{thm:offline} by union bound.
\end{proof}

\section{Proof of Lemmas in Appendix~\ref{sec:app-offline-1}}
\subsection{Proof of Lemma~\ref{lm:err-offline}}
First we need to introduce the extended value difference lemma provided in~\citet{jin2021pessimism, cai2020provably}
\begin{lemma}[Extended value difference~\citet{cai2020provably}, Lemma A.1~\citet{jin2021pessimism}]\label{lm:evd}
Let $\{\pi\}_h, \{\pi'\}_h$ by any two policies and let $\{\hat Q\}_h$ be any estimated $Q$-function. For any $h \in [H]$, define the estimated value function as $\hat V_h(s) = \hat Q_h(s, \pi_h(s))$. For all $s \in \cS$ we have
\begin{align*}
    \hat V_h(s) - V_h^{\pi'}(s) &= \sum_{h' = h}^H\EE_{\pi'}\left[\hat Q_{h'}(s_{h'}, \pi_{h'}(s_{h'})) - \hat Q_{h'}(s_{h'}, \pi'_{h'}(s_{h'}))) \middle | s_h = s\right]\\
    &\quad + \sum_{h'=h}^H\EE_{\pi'}\left [\hat Q_{h'}(s_{h'}, a_{h'}) - r(s_{h'}, a_{h'}) - [\PP \hat V_{h' + 1}](s_{h'}, a_{h'}) \middle | s_h = s\right],
\end{align*}
where $\EE_{\pi'}$ is taken with respect to the trajectory generated by $\pi'$ using underlying MDP and $a_{h'}$ is defined by $a_{h'} = \pi'_{h'}(s_{h'})$.
\end{lemma}
\begin{proof}
The proof of this lemma is same with Section B.1 in~\citet{cai2020provably} by replacing the initial state from $1$ to any arbitrary step $h$.
\end{proof}

We also provide an error control lemma similar with Lemma~\ref{lm:5} in online setting
\begin{lemma}[Lemma 5 on $B_n^{(2)}$, pp. 23,~\citet{yang2020reinforcement}]\label{lm:4}
Suppose $\cE_{\bphi}$ holds, then for any $(s, a) \in \cS \times \cA$, we have 
\begin{align*}
    \|\bphi(s, a)^\top(\Mb_{h, \bphi} - \Mb_{h, \bphi}^*)\|_2 \le \sqrt{\beta_{\bphi}\bphi^\top(s, a)\Ub_{h, \bphi}^{-1}\bphi(s, a)}.
\end{align*}
\end{lemma}
\begin{proof}
The proof is similar with~\citet{yang2020reinforcement} by fixing $k$ to $K$.
\end{proof}

Then our proof starts by following the idea in~\citet{jin2021pessimism}.
\begin{proof}
First it is obvious that $V_h^*(s) - V_h^\pi(s) = (V_h^*(s) - V_h(s)) - (V_h^\pi(s) - V_h(s))$ where $V_h$ is the estimated value function in Line~\ref{ln:5-offline} in Algorithm~\ref{alg:main-offline}. Note that $V(s) = Q(s, \pi(s))$ where $\pi$ is the output policy from Algorithm~\ref{alg:main-offline}, Lemma~\ref{lm:evd} suggests that by setting $\pi' = \pi^*$, $V^*_h(s) - V_h(s)$ can be written by
\begin{align}
    V_h(s) - V_h^*(s) &= \sum_{h' = h}^H\EE_{\pi^*}\left[ Q_{h'}(s_{h'}, \pi_{h'}(s_{h'})) - Q_{h'}(s_{h'}, \pi^*_{h'}(s_{h'}))) \middle | s_h = s\right] \notag \\
    &\quad + \sum_{h'=h}^H\EE_{\pi^*}\left [Q_{h'}(s_{h'}, a_{h'}) - r(s_{h'}, a_{h'}) - [\PP V_{h' + 1}](s_{h'}, a_{h'}) \middle | s_h = s\right] \notag \\
    &\ge \sum_{h'=h}^H\EE_{\pi^*}\left [Q_{h'}(s_{h'}, a_{h'}) - r(s_{h'}, a_{h'}) - [\PP V_{h' + 1}](s_{h'}, a_{h'}) \middle | s_h = s\right],\label{eq:i1}
\end{align}
where the last inequality is due to the fact that we are executing the greedy policy i.e. $\pi_h(s) = \argmax Q_h(s, a)$ thus $Q_h(s, \pi_h(s)) \ge Q_h(s, \pi^*_h(s))$. Meanwhile, letting $\pi = \pi' = \pi$, Lemma~\ref{lm:evd} suggests that
\begin{align}
    V_h(s) - V_h^\pi(s) = \sum_{h'=h}^H\EE_{\pi}\left [Q_{h'}(s_{h'}, a_{h'}) - r(s_{h'}, a_{h'}) - [\PP V_{h' + 1}](s_{h'}, a_{h'}) \middle | s_h = s\right].\label{eq:i2}
\end{align}
Noticing that both~\eqref{eq:i1} and~\eqref{eq:i2} are the summation about the  $Q_h(s, a) - r(s, a) - [\PP V_{h + 1}](s, a)$, which we will bound next. Recall the calculation rule of $Q$-function in Line~\ref{ln:5-offline} suggests that
\begin{align}
    &Q_h(s, a) - r(s, a) - [\PP V_{h+1}](s, a) \notag \\
    &\quad = \max_{\bphi \in \Phi} \left\{r(s, a) + \sum_{s' \in \cS} \bphi(s, a)^\top\Mb_{h, \bphi}\bpsi(s')V_{h+1}(s') - \Gamma_{h, \bphi}(s, a)\right\} - r(s, a) \notag\\
    &\qquad- \sum_{s' \in \cS}\bphi(s, a)^\top\Mb^*_{h, \bphi}\bpsi(s')V_{h+1}(s')\notag\\
    &\quad = \max_{\bphi \in \Phi} \left\{\sum_{s' \in \cS}\bphi(s, a)^\top(\Mb_{h, \bphi} - \Mb^*_{\bphi})\bpsi(s')V_{h+1}(s') - \Gamma_{h, \bphi}(s, a) \right\}\notag\\
    &\quad = \max_{\bphi \in \Phi} \left\{\bphi(s, a)^\top(\Mb_{h, \bphi} - \Mb^*_{h, \bphi})\bPsi\vb_{h+1} - C_{\bpsi}H\sqrt{\beta_{\bphi}}\|\bphi(s, a)\|_{\Ub_{h,\phi}^{-1}}\right\}\label{eq:i3}
\end{align}
where the last inequality utilize the notation $\bPsi = \left( \bpsi(s_1), \bpsi(s_2), \cdots, \bpsi(s_{|\cS|})\right)^\top \in \RR^{|\cS| \times d'}$ and $\vb_{h+1} = \left(V_{h+1}(s_1), V_{h+1}(s_2), \cdots, V_{h+1}(s_{|\cS|})\right)^\top \in \RR^{|\cS|}$. For each $\bphi \in \Phi$, lemma~\ref{lm:4} suggests that 
\begin{align*}
    \left|\bphi(s, a)^\top(\Mb_{h, \bphi} - \Mb^*_{h, \bphi})\bPsi\vb_{h+1}\right| &\le \left\|(\Mb_{h, \bphi} - \Mb^*_{h, \bphi}) \bphi(s, a)\right\|_2 \left\|\bPsi\vb_{h+1}\right\|_2 \\
    &\le C_{\bpsi}H\sqrt{\beta_{\bphi}}\|\bphi(s, a)\|_{\Ub_{h, \bphi}^{-1}},
\end{align*}
where the first inequality follows the C-S inequality and the second inequality utilizes the fact that $\|\bPsi\vb_{h+1}\|_2 \le C_{\bpsi}\|\vb_{h+1}\|_{\infty} \le C_{\bpsi}H$. Therefore for all $\bphi \in \Phi$,  for any $(s, a, h) \in \cS \times \cA \times [H]$:
\begin{align*}
    -2C_{\bpsi}H\sqrt{\beta_{\bphi}}\|\bphi(s, a)\|_{\Ub_{h, \bphi}^{-1}} \le \bphi(s, a)^\top(\Mb_{h, \bphi} - \Mb^*_{h, \bphi})\bPsi\vb_{h+1} - C_{\bpsi}H\sqrt{\beta_{\bphi}}\|\bphi(s, a)\|_{\Ub_{h,\phi}^{-1}} \le 0.
\end{align*}
Plugging this into~\eqref{eq:i3} yields
\begin{align*}
    &Q_h(s, a) - r(s, a) - [\PP V_{h+1}](s, a) \\
    &\quad= \max_{\bphi \in \Phi} \left\{\bphi(s, a)^\top(\Mb_{h, \bphi} - \Mb^*_{h, \bphi})\bPsi\vb_{h+1} - C_{\bpsi}H\sqrt{\beta_{\bphi}}\|\bphi(s, a)\|_{\Ub_{h,\phi}^{-1}}\right\} \\
    &\quad\ge \max_{\bphi \in \Phi}\left\{-2C_{\bpsi}H\sqrt{\beta_{\bphi}}\|\bphi(s, a)\|_{\Ub_{h, \bphi}^{-1}}\right\}\\
    &\quad = -2C_{\bpsi}H\min_{\bphi \in \Phi}\left\{\sqrt{\beta_{\bphi}}\|\bphi(s, a)\|_{\Ub_{h, \bphi}^{-1}}\right\}
\end{align*}
and $Q_h(s, a) - r(s, a) - [\PP_h V_{h+1}](s, a) \le 0$ for all $(s, a) \in \cS \times \cA$. Plugging the bound back to~\eqref{eq:i1} yields
\begin{align}
    V_h(s) - V_h^*(s) &\ge \sum_{h'=h}^H\EE_{\pi^*}\left [Q_{h'}(s_{h'}, a_{h'}) - r(s_{h'}, a_{h'}) - [\PP V_{h' + 1}](s_{h'}, a_{h'}) \middle | s_h = s\right] \notag \\
    &\quad\ge -2C_{\bpsi}H\sum_{h' = h}^H\EE_{\pi^*}\left[\min_{\bphi \in \Phi}\left\{\sqrt{\beta_{\bphi}}\|\bphi(s, a)\|_{\Ub_{h', \bphi}^{-1}}\right\}\middle | s_h = s\right]\label{eq:ii1}
\end{align}
and back to~\eqref{eq:i2} yields
\begin{align}
    V_h(s) - V_h^{\pi}(s) & =  \sum_{h'=h}^H\EE_{\pi^*}\left [Q_{h'}(s_{h'}, a_{h'}) - r(s_{h'}, a_{h'}) - [\PP V_{h' + 1}](s_{h'}, a_{h'}) \middle | s_h = s\right] \le 0\label{eq:ii2}.
\end{align}
substituting~\eqref{eq:ii1} from~\eqref{eq:ii2} yields the claimed result:
\begin{align*}
    V^*_h(s) - V_h^\pi(s) \le 2C_{\bpsi}H\sum_{h' = h}^H\EE_{\pi^*}\left[\min_{\bphi \in \Phi}\left\{\sqrt{\beta_{\bphi}}\|\bphi(s, a)\|_{\Ub_{h', \bphi}^{-1}}\right\}\middle | s_h = s\right]
\end{align*}
\end{proof}
\subsection{Proof of Lemma~\ref{lm:sublin-offline}}
\begin{proof}
First we show that the covariance matrix $\Ub_{\bphi, h}$ is almost linearly growth with respect to the expectation $\EE_{d_h^{\hat \pi}}[\bphi\bphi^\top]$ and the size of offline data. Considering the formalization of covariance matrix
\begin{align*}
    \Ub_{\bphi, h} &= \Ib + \sum_{(s, a, s') \in \cD_h} \bphi(s, a)\bphi^\top(s, a) \\
    &= \Ib - \sum_{(s, a, s') \in \cD_h} \underbrace{\EE_{d_h^{\hat \pi}}[\bphi(s, a)\bphi^\top(s, a)] - \bphi(s, a)\bphi^\top(s, a)}_{\bepsilon_h} + |\cD_h|\EE_{d_h^{\hat \pi}}[\bphi(s, a)\bphi^\top(s, a)].
\end{align*}
One can verify that $\EE[\bepsilon_h] = \zero$ where the expectation is taken with respect to the randomness in the generation of the offline data. Since we have $\|\bphi(s, a)\|_2^2 \le C_{\bphi}d_{\bphi}$, it is obvious that $\|\bepsilon_h\|_2 \le \|\bphi\bphi^\top\|_2 + \|\EE[\bphi\bphi^\top]\|_2 \le 2\|\bphi\|_2^2 \le 2C_{\bphi}d_{\bphi}$ by triangle's inequality. Then by Lemma~\ref{lm:matazuma}, with probability at least $1 - \delta$, for $|\cD_h| = K$ data,
\begin{align*}
    \lambda_{\max} \left(\sum_{(s, a, s') \in \cD_h} \bepsilon_h\right) \le 4C_{\bphi}d_{\bphi}\sqrt{2K\log(d_{\bphi}/\delta)}.
\end{align*}
Therefore, it's suffice to show that
\begin{align}
    \Ub_{\bphi, h} \succeq K\EE_{d_h^{\hat \pi}}[\bphi(s, a)\bphi^\top(s, a)] + \left(1 - 4C_{\bphi}d_{\bphi}\sqrt{2K\log(d_{\bphi} / \delta)}\right)\Ib \label{eq:u1}
\end{align}

Furthermore, noticing that $\EE_{d_h^{\hat \pi}}[\bphi(s, a)\bphi^\top(s, a)]$ can be always written as 
\begin{align}
    \EE_{d_h^{\hat \pi}}[\bphi(s, a)\bphi^\top(s, a)] = \Qb^\top \text{diag}(\db_r, \zero_{d_{\bphi} - r})\Qb, \label{eq:u5}
\end{align}
where $\Qb$ is an orthogonal matrix, $\db_r \in \RR^r$ is the non-zero eigen values of $\EE_{d_h^{\hat \pi}}[\bphi(s, a)\bphi^\top(s, a)]$ its minimal element as $\tilde \sigma_{h, \bphi}$. $r = \text{rank}(\EE_{d_h^{\hat \pi}}[\bphi(s, a)\bphi^\top(s, a)])$. Then~\eqref{eq:u1} can be formalized as
\begin{align}
    \Qb \Ub_{\bphi, h} \Qb^\top \succeq \text{diag}\left(K\db_r - a\mathbf 1_r + \mathbf 1_r, -a \mathbf 1_{d_{\bphi} -r} + \mathbf 1_{d_{\bphi} - r} \right), \label{eq:u2}
\end{align}
where $a = 4C_{\bphi}d_{\bphi}\sqrt{2K \log(d_{\bphi} / \delta)}$ is in the order of $\sqrt K$. 
On the other hand, since $\Ub_{\bphi, h} \succeq \Ib$, then $\Qb\Ub_{\bphi, h}\Qb^\top \succeq \Qb\Qb^\top = \Ib$. Combining this with~\eqref{eq:u2} we can conclude that
\begin{align}
    \Qb\Ub_{\bphi, h}\Qb^\top \succeq \text{diag} \left(K\db_{r} - a\mathbf 1_r + \mathbf 1_r, \mathbf 1_{d_{\bphi} - r}\right).\label{eq:u3}
\end{align}
Noticing the minimal element of $\db_r$ is $\tilde \sigma_{h, \bphi}$, then when $K\tilde \sigma_{h, \bphi} \ge a$, which we will verify later, the RHS of~\eqref{eq:u3} is positive definite, which implies that 
\begin{align}
    \Qb\Ub^{-1}_{\bphi, h}\Qb^\top \preceq \text{diag} \left(K\db_{r} - a\mathbf 1_r + \mathbf 1_r, \mathbf 1_{d_{\bphi} - r}\right)^{-1}. \label{eq:u4}
\end{align}

By union bound~\eqref{eq:u4} holds for all $\bphi \in \Phi$ and $h \in [H]$ with probability at least $1 - H|\Phi|\delta$. 
Then for any $(s, a, h) \in \cS \times \cA \times [H]$, according to Assumption~\ref{asm:offline}, there exists $\bphi \in \Phi$ and $\yb \in \RR^d$ such that $\bphi(s, a) = \EE_{d_h^{\hat \pi}}[\bphi(s, a)\bphi^\top(s, a)] \yb$, combining this with~\eqref{eq:u5} and \eqref{eq:u4} yields that 
\begin{align}
    &\bphi^\top(s, a) \Ub_{h, \bphi}^{-1} \bphi(s, a) \notag \\
    &\quad= \yb^\top \EE_{d_h^{\hat \pi}}[\bphi(s, a)\bphi^\top(s, a)]^\top \Ub_{h, \bphi}^{-1}\EE_{d_h^{\hat \pi}}[\bphi(s, a)\bphi^\top(s, a)] \yb \notag \\
    &\quad= \yb^\top \Qb^\top \diag(\db_r, \zero_{d_{\bphi} - r}) \Qb \Ub_{h, \bphi}^{-1} \Qb^\top \diag(\db_r, \zero_{d_{\bphi} - r}) \Qb \yb \notag \\
    &\quad\le \yb^\top \Qb^\top \diag(\db_r, \zero_{d_{\bphi} - r})\text{diag}(K\db_r - a\mathbf 1_r + \mathbf 1_r, \mathbf 1_{d_{\bphi} - r})^{-1}\diag(\db_r, \zero_{d_{\bphi} - r}) \Qb\yb \label{eq:u6},
\end{align}
where the last inequality follows~\eqref{eq:u4}. Noticing that $\text{diag}(\db_r, \zero_{d_{\bphi - r}})\Qb \yb$ can be written as $\text{diag}(\db_r, \zero_{d_{\bphi - r}})\Qb \yb = \left(\zb_r^\top, \zero_{d_{\bphi} - r}^\top\right)^\top$. Therefore~\eqref{eq:u6} becomes
\begin{align}
    \bphi^\top(s, a) \Ub_{h, \bphi}^{-1} \bphi(s, a) &\le \left(\zb_r^\top, \zero_{d_{\bphi} - r}^\top\right)^\top\text{diag}(K\db_r - a\mathbf 1_r + \mathbf 1_r, \mathbf 1_{d_{\bphi} - r})^{-1}\left(\zb_r^\top, \zero_{d_{\bphi} - r}^\top\right)^\top \notag \notag \\
    &= \zb_r^\top \diag(K\db_r - a\mathbf 1_r + \mathbf 1_r)^{-1}\zb_r. \label{eq:u7}
\end{align}
Noticing that 
\begin{align*}
    \|\zb_r\|_2 = \left\|\text{diag}(\db_r, \zero_{d_{\bphi} - r}) \Qb \yb\right\| = \left\|\Qb\bphi(s, a)\right\|_2 = \left\|\bphi(s, a)\right\|_2 \le \sqrt{C_{\bphi}d_{\bphi}},
\end{align*}
where the last equality comes form Definition~\ref{asm:lin}, \eqref{eq:u7} finally becomes 
\begin{align}
    \bphi^\top(s, a) \Ub_{h, \bphi}^{-1} \bphi(s, a) \le \zb_r^\top \diag(K\db_r - a\mathbf 1_r + \mathbf 1_r)^{-1}\zb_r \le \frac{C_{\bphi}d_{\bphi}}{K\tilde \sigma_{h, \bphi} - a + 1} \le \frac{C_{\bphi}d_{\bphi}}{K\tilde \sigma_{h, \bphi} - a}, \label{eq:c1}
\end{align}
where the second last inequality is due to $\lambda_{\max} \left(\diag(K\db_r - a\mathbf 1_r + \mathbf 1_r)^{-1}\right) = (K\tilde \sigma_{h, \bphi} - a + 1)^{-1}$ and the last inequality utilizes the assumption that $K\tilde \sigma_{h, \bphi} \ge a 0$

Next, to control $\|\bphi(s, a)\|_{\Ub_{h, \bphi}^{-1}} \le \Delta / (2 C_{\bpsi} H^2\sqrt{\beta_{\bphi}}) =: B$, by~\eqref{eq:c1}, it suffices to control 
\begin{align}
    \frac{C_{\bphi}d_{\bphi}}{K\tilde \sigma_{h, \bphi} - 4C_{\bphi}d_{\bphi}\sqrt{2K\log(d_{\bphi}/ \delta)}} \le B^2.\label{eq:c2}
\end{align}
Denoting $\sqrt{K} = 4C_{\bphi}d_{\bphi}\sqrt{2\log(d_{\bphi} / \delta)}\tilde \sigma_{h, \bphi}^{-1} x$, then the constrain $K \tilde \sigma_{h, \bphi} \ge 4C_{\bphi}d_{\bphi}\sqrt{2K\log(d_{\bphi} / \delta)}$ is equivalent with $x \ge 1$, and~\eqref{eq:c2} becomes 
\begin{align}
    \frac{C_{\bphi}d_{\bphi}}{B^2} \le \frac{32C^2_{\bphi}d^2_{\bphi}\log(d_{\bphi} / \delta)}{\tilde \sigma_{h, \bphi}}(x^2 - x). \label{eq:c3}
\end{align}
Since the sufficient condition of inequality $x^2 - x - c > 0$ is $x > (1 + \sqrt{1 + 2c}) / 2$ which could be implied by $x > \sqrt{1 + 2c}$ by C-S inequality, the sufficient condition of~\eqref{eq:c3} could be written as
\begin{align}
    x > \sqrt{1 + \frac{C_{\bphi}d_{\bphi}\tilde \sigma_{h, \bphi}}{16B^2C_{\bphi}^2d_{\bphi}^2\log(d_{\bphi} / \delta)}}, \label{eq:c4}
\end{align}
which can imply the constrain that $x \ge 1$. Plugging the notations $B := \Delta / (2C_{\bpsi}H^2\sqrt{\beta_{\bphi}})$ and $\sqrt{K} = 4C_{\bphi}d_{\bphi}\sqrt{2\log(d_{\bphi} / \delta)}\tilde \sigma_{h, \bphi}^{-1} x$ back into~\eqref{eq:c4} yields for all $(s, a, h) \in \cS \times \cA \times [H]$, with probability at least $1 - H|\Phi|\delta$, there exists $\bphi \in \Phi$ such that when  
\begin{align*}
    K > \frac{32C_{\bphi}^2d_{\bphi}^2\log(d_{\bphi}/\delta)}{\tilde \sigma^{2}_{h, \bphi}}\left(1 + \frac{C_{\bpsi}^2H^4\beta_{\bphi}C_{\bphi}\tilde \sigma_{h, \bphi}}{4\Delta^2C_{\bphi}^2d_{\bphi}\log(d_{\bphi} / \delta)}\right),
\end{align*}
$\|\bphi(s, a)\|_{\Ub_{h, \bphi}} < \gap{\min} / (2C_{\bpsi}H^2\sqrt{\beta_{\bphi}})$. Replacing $\delta$ by $\delta / (H|\Phi|)$ we can get the claimed result.
\end{proof}
\end{document}